\documentclass{article}
\usepackage[left=3.4cm,right=3.4cm,bottom=3.5cm,top=3.4cm]{geometry}
\usepackage{graphicx} 
\graphicspath{ {images/} }
\usepackage{amsmath,amssymb,fancybox,ascmac,color, amsthm}
\usepackage{bm}
\usepackage{algorithm}
\usepackage{algorithmic}
\usepackage{threeparttable}

\usepackage{multirow}
\usepackage[normalem]{ulem}
\DeclareMathOperator*{\argmin}{argmin}
\DeclareMathOperator*{\argmax}{argmax}
\usepackage{hyperref}
\hypersetup{
 setpagesize=false,
 colorlinks=true,
 linkcolor=blue,
 citecolor=blue}

\newtheorem{lemma}{Lemma}
\newtheorem{corollary}{Corollary}
\newtheorem{definition}{Definition}
\newtheorem{theorem}{Theorem}

\makeatletter
\def\@seccntformat#1{\@ifundefined{#1@cntformat}%
   {\csname the#1\endcsname\quad}  
   {\csname #1@cntformat\endcsname}
}
\let\oldappendix\appendix 
\renewcommand\appendix{%
    \oldappendix
    \newcommand{\section@cntformat}{\appendixname~\thesection\quad}
}
\makeatother

\usepackage[svgnames]{xcolor}
\usepackage{listings}
\lstdefinestyle{mystyle}{
  columns=fixed,  
  basewidth=0.5em,  
  lineskip=-3pt,  
  frame=lines,  
  breaklines=true,             
  columns=[l]{fullflexible},   
  numbers=left,                
  numberstyle={\scriptsize},   
  stepnumber=1,                
  basicstyle={\ttfamily\small},  
  keywordstyle=[1]{\color{RoyalBlue}},  
  keywordstyle=[2]{\color{VioletRed}},  
  stringstyle={\color{FireBrick}},  
  commentstyle={\color{SeaGreen}},  
}
\title{Dose-finding design based on level set estimation in phase I cancer clinical trials}

\author{Keiichiro Seno$^{1a}$, Kota Matsui$^{2b\ast}$, Shogo Iwazaki$^3$,\\
Yu Inatsu$^4$, Shion Takeno$^{5, 6}$ and Shigeyuki Matsui$^{2, 7}$}
\date{}

\begin{document}

\maketitle

\noindent
$^1$~Department of Biostatistics, Nagoya University

\noindent
$^2$~Department of Biostatistics, Kyoto University

\noindent
$^3$~MI-6 Ltd.

\noindent
$^4$~Department of Computer Science, Nagoya Institute of Technology

\noindent
$^5$~Department of Mechanical Systems Engineering, Nagoya University

\noindent
$^6$~Center for Advanced Intelligence Project, RIKEN

\noindent
$^7$~Research Center for Medical and Health Data Science, The Institute of Statistical Mathematics

\footnote{\hspace{-18pt}  
a:senokei.bs@gmail.com \\ 
b:matsui.kota.8k@kyoto-u.ac.jp \\ 
$\ast$ corresponding author}

\begin{abstract}
\noindent
The primary objective of phase I cancer clinical trials is to evaluate the safety of a new experimental treatment and to find the maximum tolerated dose (MTD).
We show that the MTD estimation problem can be regarded as a level set estimation (LSE) problem whose objective is to determine the regions where an unknown function value is above or below a given threshold.
Then, we propose a novel dose-finding design in the framework of LSE.
The proposed design determines the next dose on the basis of an acquisition function incorporating uncertainty in the posterior distribution of the dose-toxicity curve as well as overdose control.
Simulation experiments show that the proposed LSE design achieves a higher accuracy in estimating the MTD and involves a lower risk of overdosing allocation compared to existing designs, thereby indicating that it provides an effective methodology for phase I cancer clinical trial design.
\end{abstract}

\section{Introduction}
The primary objective of phase I cancer clinical trials of new experimental treatments is to estimate the maximum tolerated dose (MTD), defined as the maximum dose at which the probability of unacceptable adverse events, referred to as the dose-limiting toxicity (DLT), is below a target level, such as 20\% or 30\%.
The MTD estimate serves as the basis for determining the recommended dose in subsequent clinical trials.
The various phase I clinical trial designs that have been proposed to estimate the MTD can be categorized into algorithm-based, model-based, and model-assisted designs \cite{yuan2019model}.

Among algorithm-based designs, the traditional 3+3 design \cite{storer1989design} has been widely used because of its simple dose-escalation rule. 
However, this design is known to achieve low accuracy in estimating the MTD mainly because of its memoryless dose-escalation rule, and many patients may be treated with subtherapeutic doses \cite{le2009dose}.
To improve the estimation efficiency, model-based designs, such as the continual reassessment method (CRM) \cite{o1990continual} and its extensions \cite{babb1998cancer,neuenschwander2008critical}, have been proposed.
These methods assume a parametric model for the dose-toxicity relationship, sharing information on toxicity across doses and allowing the selection of the recommended dose for the next patient cohort on the basis of all the toxicity data accumulated during the trial.
However, misspecification of the parametric model for the dose-toxicity relationship may lead to the estimation of the wrong dose as the MTD.
To ensure model robustness, model-assisted designs, such as the modified toxicity probability interval (mTPI) design \cite{ji2010modified}, the Keyboard design \cite{yan2017keyboard}, and the Bayesian optimal interval (BOIN) design \cite{liu2015bayesian}, have been proposed. 
By employing marginal conjugate models such as beta-binomial models for each dose without specifying the dose-toxicity curves across doses, these designs aim to provide a simple dose-escalation rule without compromising the performance in dose escalation by incorporating the toxicity data accumulated during the trial.
Meanwhile, owing to the absence of any model for information sharing across doses, separate MTD estimation, such as isotonic regressions, may be required to estimate the MTD at the end of the trial. 

Both model-based and model-assisted designs adaptively select the next dose on the basis of the accumulated data during the trial. 
These designs can be regarded as one in the more general framework of adaptive experimental design, which is a generic term for methods that efficiently make inferences about the properties of interest for unknown systems or unknown functions that are expensive to evaluate, by performing experiments in an adaptive manner \cite{Garud2017-zd,Liu2018-ql}. 
Recently, a dose-finding design was proposed on the basis of Bayesian optimization (BO), which is one of the methods of adaptive experimental design \cite{Takahashi2021-MTD,takahashi2021-MTDC,takahashi2021-OBD}. 
The goal of BO is to find the point that takes the maximum or minimum value of an unknown function in the fewest possible evaluations \cite{garnett_bayesoptbook_2023}. 
Takahashi and Suzuki \cite{Takahashi2021-MTD} proposed a dose-finding design in phase I cancer clinical trials using the BO algorithm, which will be hereafter referred to as the BO design.
Based on a Bayesian nonparametric model using a Gaussian process (GP) \cite{williams2006gaussian} for the dose-toxicity curve, the BO design seeks to identify a dose that minimizes the value of the objective function, defined as the absolute error between the DLT probability at each dose and the target threshold.
In contrast to many model-based and model-assisted designs with dose selection based on point estimation, the BO design incorporates the uncertainty of the posterior of the dose-toxicity curve across all the doses as well as some safety constraints in its dose-selection strategy.

In this paper, we propose a novel dose-finding design based on the level set estimation (LSE), namely LSE design. 
We show that the MTD estimation problem can be considered as an LSE problem that seeks to divide an input space into two regions where an unknown function value is above or below a given threshold.
To solve this problem efficiently, we use the framework of active learning for LSE \cite{Gotovos2013}, which has recently been applied to various fields such as environmental monitoring, psychophysics and materials engineering \cite{Gotovos2013,owen2021adaptive,hozumi2023adaptive}.
In this framework, it is possible to divide the input space into two regions with the fewest possible observations.
After dividing the dose space into two regions according to whether the toxicity probability is above or below the target threshold, we expect that the MTD lies near the boundary between these two regions.
On the basis of a nonparametric Gaussian process model across doses, we propose an acquisition function for the dose selection strategy that incorporates uncertainty in the posterior distribution of the dose-toxicity curve as well as overdose control.
In addition, we propose a prior specification for applying our LSE design to dose-finding trials.
Notably, in contrast to many previous dose-finding designs, the proposed LSE design provides an estimate of the MTD (as the boundary dose between the two regions) that may fall between the candidate doses to be tested.
This estimate may be useful not only for detecting incorrect dosing during the trial but also for determining the recommended dose for subsequent trials, especially when the experimental treatment dose can be adjusted.

In a standard LSE problem, the objective evaluation function (which corresponds to the dose-toxicity curve in this study) typically takes continuous values over the entire real number space, and theoretical guarantees of convergence have been established for this case. 
Meanwhile, as will be discussed in detail later, in the dose-finding problem, the output of the evaluation function is a probability value in the range of $[0,1]$, which makes it impossible to directly apply conventional theoretical evaluations. 
Therefore, we have demonstrated that in the dose-finding problem setting, the same type of convergence guarantee as that in conventional LSE can be established. 
This constitutes the theoretical contribution of our study.

The remainder of this paper is organized as follows. Section \ref{sec: Background} provides an overview of the proposed LSE framework for dose-finding trials.
Section \ref{sec: Theory} discusses the theoretical convergence guarantee of LSE in the dose-finding problem.
Section \ref{sec: Methods} presents the details of the proposed design, including the specification of the acquisition function and the prior specification. 
Section \ref{sec: Numerical results} describes simulations conducted to compare the performances of the proposed and existing designs. Finally, Section \ref{sec: Discussion} concludes the paper.

\section{Overview of LSE for dose finding}\label{sec: Background}

\subsection{LSE}\label{sec: level set estimation}
For an unknown function $f : \mathcal{X} \rightarrow \mathbb{R}$ and a given threshold $\theta \in \mathbb{R}$, LSE is a problem of dividing the input space $\mathcal{X}$ into the following two sets:
\begin{align*}
    \text{superlevel set}~~~~ H &= \{x\in\mathcal{X} \mid f(x) > \theta\},\\
    \text{sublevel set}~~~~   L &= \{x\in\mathcal{X} \mid f(x) \leq \theta\}.
\end{align*}
We consider a sequential design for estimating the superlevel/sublevel sets. 
At each step $t$, we select an evaluation point $x_t\in\mathcal{X}$ based on all the previous data and observe the output $y_t$.
In the case where the output is a continuous value, it is natural to assume a noisy output $y_t = f(x_t) + \eta_t$.

In situations where obtaining the output is expensive, it is desirable to reduce the number of observations of the data used to solve the LSE problem to the fewest possible.
Several methods that actively select the next evaluation point have been proposed to achieve this objective \cite{Bryan2005, Gotovos2013}.
These methods often assume that the unknown function is a sample from a Gaussian process.
This is because the Gaussian process allows $f$ to be modeled in a nonparametric manner and quantifies the uncertainty in the estimation, thereby facilitating efficient estimation of the level sets.
A Gaussian process is a collection of random variables, any finite subset of which follows a multivariate normal distribution \cite{williams2006gaussian}.
A Gaussian process is denoted by $\mathrm{GP}(m,k)$ and is specified by its mean function $m(x)$ and covariance function $k(x,x')$.
Suppose that $f$ is a sample from $\mathrm{GP}(m,k)$; then, any collection of function values $(f(x_1),\dots,f(x_n))^{\top}$ follows a multivariate normal distribution $N(\bm{m},\bm{K})$, where $\bm{m} = (m(x_1),\dots,m(x_n))^{\top}$ and $\bm{K}_{ij} = k(x_i,x_j)$.

Now, suppose that we assume a Gaussian process prior for $f$, i.e., $f \sim \mathrm{GP}(m,k)$, and that we have $n$ observations $y_t = f(x_t) + \eta_t$, where $\eta_t \sim N(0,\sigma^2)$ for $t=1,\dots,n$.
Then, for a new input $x_{\ast}$, the posterior distribution of $f(x_{\ast})$ becomes a normal distribution $f(x_{\ast}) \sim N(\mu_n(x_{\ast}), \sigma_n^2(x_{\ast}))$, where 
\begin{align}
\label{eq:posterior_mean}
\mu_n(x_{\ast})&=m(x_{\ast})+\bm{k}(x_{\ast})^{\top}\left(\bm{K}+\sigma^2\bm{I}_n\right)^{-1}(\bm{y}-\bm{m}),\\   
\label{eq:posterior_variance}
\sigma_n^2(x_{\ast})&=k(x_{\ast},x_{\ast})-\bm{k}(x_{\ast})^{\top}\left(\bm{K}+\sigma^2\bm{I}_n\right)^{-1}\bm{k}(x_{\ast}).
\end{align}
Furthermore, $\bm{k}(x_{\ast}) = (k(x_{\ast}, x_1), ..., k(x_{\ast}, x_n))^{\top}$, $\bm{I}_n$ is the $n$-dimensional identity matrix, and $\bm{y} = (y_1, ..., y_n)^{\top}$ is the observation vector. 
Using $\mu_n(x)$ and $\sigma_n^2(x)$, we can construct the credible interval for $f(x)$ as 
\begin{align}
    \label{eq:credible_interval}
    C_n(x) = [\mu_n(x) - \beta_n^{1/2}\sigma_n(x), \mu_n(x) + \beta_n^{1/2}\sigma_n(x)], 
\end{align}
where $\beta_n > 0$ is a hyperparameter that determines the width of the interval.

Then, we can construct a classification rule that determines whether each $x \in \mathcal{X}$ is classified into $H$ or $L$ as follows~\cite{Gotovos2013}:
\begin{align}
    \label{eq:classification_rule}
    x 
    \in 
    \begin{cases}
        H & \mbox{ if } \min C_n(x) + \xi = \mu_n(x) - \beta_n^{1/2}\sigma_n(x) + \xi > \theta, \\ 
        L & \mbox{ if } \max C_n(x) - \xi = \mu_n(x) + \beta_n^{1/2}\sigma_n(x) - \xi < \theta,  
    \end{cases}
\end{align}
where $\xi > 0$ is a hyperparameter that defines the margin for the classification rule.
If $x$ is not classified into either $H$ or $L$, it is classified into the undetermined set, $U$.
Then, the next evaluation point is determined to be the input point $x\in\mathcal{X}$ that maximizes the acquisition function $\alpha(x)$.
The details of the acquisition function will be discussed later.

Based on the above, the overall algorithm of LSE is constructed as shown in Algorithm~\ref{alg: ALSE algorithm}.

\begin{algorithm}[t]
\caption{LSE}
\label{alg: ALSE algorithm}
\begin{algorithmic}
\REQUIRE observed data $D_n$, a set of candidate points $\mathcal{X}$, and hyperparameters $\beta_n, \xi > 0$. 
\ENSURE estimated superlevel set $H$ and sublevel set $L$
\STATE {\bfseries Initialize:} $H, L \leftarrow \emptyset$, $U \leftarrow \mathcal{X}$
\WHILE{$U \neq \emptyset$}
\STATE{\bfseries Step~1} Construct the credible interval $C_n(x)$ using \eqref{eq:posterior_mean}–\eqref{eq:credible_interval}.
\STATE{\bfseries Step~2} 
Based on \eqref{eq:classification_rule}, each candidate point $x \in \mathcal{X}$ is classified into either the superlevel set $H$ or the sublevel set $L$. 
Inputs $x$ that do not satisfy the condition of \eqref{eq:classification_rule} are classified into the undetermined set $U$. 
\STATE{\bfseries Step~3} Select the next evaluation point by maximizing the acquisition function as $x_{n+1} = \argmax_{x\in \mathcal{X}} \alpha_n(x)$ and observe $y_{n+1} = f(x_{n+1}) + \eta_{n+1}$, $\eta_{n+1} \sim N(0, \sigma^2)$. 
\STATE{\bfseries Step~4} Update the dataset as $D_{n+1} \leftarrow D_n \cup \{(x_{n+1}, y_{n+1})\}$ and $n \leftarrow n+1$. 
\ENDWHILE
\end{algorithmic}
\end{algorithm}

\subsection{Dose finding based on LSE}
In this section, we first show that dose-finding trials can be formulated as an LSE problem and then outline our dose-finding design based on the LSE framework.
In the context of dose-finding trials, we define $\pi(x)$ as an unknown function of the toxicity probability of a DLT occurring at dose $x\in\mathcal{X}\subset\mathbb{R}$.
As the toxicity increases with the dose, $\pi(x)$ is assumed to be a monotonically increasing function.
For the target DLT probability $\theta$ as a threshold, we consider the LSE problem in which the objective is to divide the dose space into a sublevel set $L = \{x\in\mathcal{X} \mid \pi(x) \leq \theta\}$ and a superlevel set $H = \{x\in\mathcal{X} \mid \pi(x) > \theta\}$.
If there exists a dose $x^*\in\mathcal{X}$ such that $\pi(x^*)=\theta$, then under the monotonicity  assumption of $\pi(x)$, the dose space is divided into $L$ and $H$ with $x^* = \max_{x\in\mathcal{X}} L$ as the boundary dose (Figure \ref{fig: LSE}).
The dose $x^*$ is the MTD that we try to estimate; thus, we can estimate the MTD by solving the aforementioned LSE problem.

Although only a limited number of discrete candidate doses can be investigated in the clinical trial, estimating $L$ and $H$ enables us to estimate the MTD (that is not necessarily one of the discrete candidate doses) and select a recommended dose from these discrete candidate doses.
Various criteria can be considered for selecting a recommended dose.
For example, a recommended dose could be selected as the dose that takes the highest posterior probability $\mathrm{Pr}(|\pi(d)-\theta|\leq\delta \mid \text{data})$, where $\delta>0$ is a small margin, between the highest dose in the estimated sublevel set and the lowest dose in the estimated superlevel set.
Alternatively, with a conservative policy, it may be better to select the highest dose in the estimated sublevel set as a recommended dose for the MTD.

\begin{figure}[htbp]
\begin{center}
\includegraphics[width=0.47\linewidth]{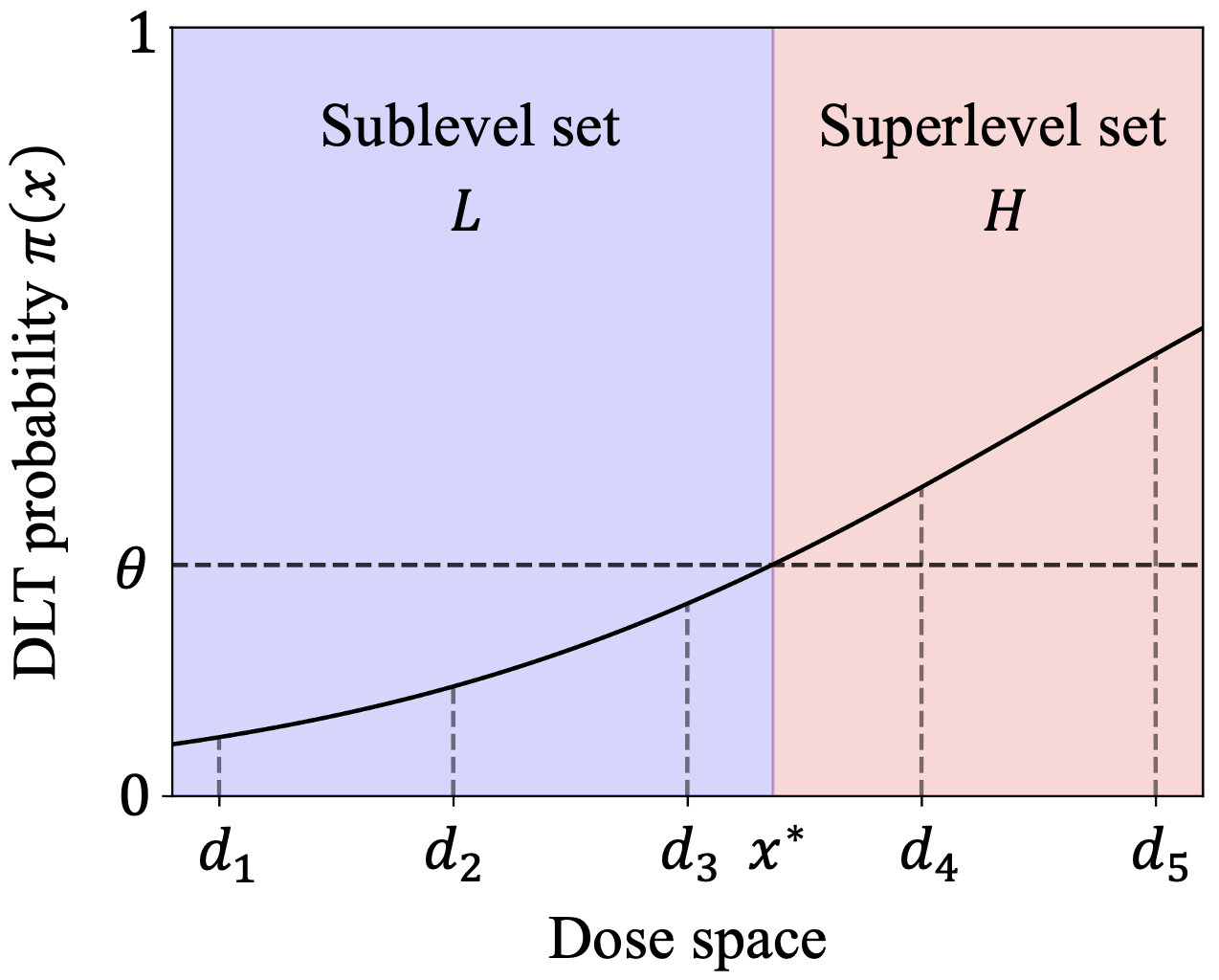}
\caption{Illustration of LSE for dose finding with five candidate doses. The solid curve represents the DLT probability $\pi(x)$, and the dotted horizontal line represents the threshold, i.e., the target DLT probability $\theta$. Under the monotonicity assumption for $\pi(x)$, the dose space is divided into $L$ and $H$ by the dose $x^* = \max_{x\in\mathcal{X}} L$ as the boundary.}
\label{fig: LSE}
\end{center}
\end{figure}

\section{Theoretical justification of LSE for dose-finding design}\label{sec: Theory}
In conventional LSE, the output of the black-box function is assumed to be a continuous quantity that can take any value on $\mathbb{R}$, and a theoretical convergence analysis based on direct modeling using Gaussian processes has been proposed~\cite{Gotovos2013}. 
Meanwhile, the output of the black-box function $\pi$ considered in this study is a probability that takes values in $[0, 1]$. 
Moreover, the observable output is a binary variable that indicates whether toxicity manifests in a patient at a given dose, and the value of $\pi$ is not provided explicitly. 
Therefore, it is necessary to incorporate the clinical trial design after confirming that the theory of conventional LSE also holds in our setting. 

In what follows, we present theoretical results for the LSE problem in the setting of dose-finding trials that are analogous to those obtained in the conventional LSE setting. 
Specifically, we guarantee that the LSE algorithm always stops under a finite sample size and that the estimated level sets obtained when the LSE algorithm stops approximate the true level set sufficiently well.
%

\subsection{Preliminaries}
We consider the problem of estimating the level set of a fixed DLT probability function $\pi : \mathcal{X} \rightarrow [0, 1]$ by a threshold $\theta$, where $\mathcal{X} \subset \mathbb{R}$ is a compact set.  
%
For the $i$-th patient, a dose $x_i \in \mathcal{X}$ is chosen; then, the DLT outcome (binary label) is observed as $y_i \sim \mathrm{Bernoulli}(\pi(x_i))$, $i = 1, ..., n$. 
The model for observation $y_i$ can be interpreted as a noisy observation of $\pi$ as follows:
\begin{align}
    \label{eq:sub_gaussian_noise}
    y_i = \pi(x_i) + \varepsilon_i(x_i), 
    ~~~
    \varepsilon_i(x_i)
    = \begin{cases}
    1 - \pi(x_i) & \mbox{with probability} ~ \pi(x_i) \\ 
    - \pi(x_i) & \mbox{with probability} ~ 1 - \pi(x_i)
    \end{cases}
    .
\end{align}
Since $\varepsilon_i(x_i) \in [- \pi(x_i), 1 - \pi(x_i)]$, it is bounded; thus, given $x_i$, it is a conditionally sub-Gaussian random variable with a variance proxy of $1$. 

\paragraph{Regularity assumptions.}
Suppose that there exists an unknown function $\tilde{\pi} : \mathcal{X} \rightarrow [-0.5, 0.5]$ such that $\pi(x) = \tilde{\pi}(x) + 0.5$ for any $x \in \mathcal{X}$. 
Then, we can consider the modified outcome $\tilde{y}_i = y_i - 0.5$ for $\tilde{\pi}$. 
We  assume that $\mathcal{X}$ is endowed with a positive definite kernel $k:\mathcal{X} \times \mathcal{X} \rightarrow \mathbb{R}$ that is normalized to satisfy $k(x, x^{\prime}) \le 1$ for any $x, x^{\prime} \in \mathcal{X}$. 
We also assume that $\tilde{\pi}$ has a bounded norm in the corresponding reproducing kernel Hilbert space (RKHS) $\mathcal{H}_k(\mathcal{X})$, i.e., $\tilde{\pi} \in \mathcal{H}_k(\mathcal{X})$ and there exists a constant $B_{\tilde{\pi}} > 0$ such that $\|\tilde{\pi}\|_k \le B_{\tilde{\pi}}$ holds. 

\subsection{Gaussian process surrogate model for $\tilde{\pi}$}
In the following, data with the modified outcome up to the $n$-th patient are denoted as $\tilde{D}_n = \{(x_i, \tilde{y}_i)\}_{i=1}^n$.
In standard LSE settings~\cite{Gotovos2013}, Bayesian inference is often performed by assuming Gaussian noise $\tilde{\varepsilon}_i = \tilde{y}_i - \tilde{\pi}(x_i)$ and introducing a Gaussian process prior with mean $0$ and a positive definite kernel $k$ for $\tilde{\pi}$. 
With such a surrogate model, given $n$ input-output pairs $\tilde{D}_n$, the posterior distribution for $\tilde{\pi}$ is also Gaussian with mean and variance 
\begin{align}
    \label{eq:gaussian_process_posterior}
    \begin{aligned}
    \tilde{\mu}_n(x) &= \bm{k}_n(x)^{\top}(\bm{K}_n + \lambda \bm{I}_n)^{-1} \tilde{\bm{y}}_n, \\ 
    \tilde{\sigma}_n^2(x) &= k(x, x) - \bm{k}_n(x)^{\top}(\bm{K}_n + \lambda \bm{I}_n)^{-1}\bm{k}_n(x),
    \end{aligned}
\end{align}
where $\bm{k}_n(x) = (k(x_1, x), ..., k(x_n, x))^{\top}$ and $\bm{K}_n = (k(x_i, x_j))_{i, j}$ is the kernel matrix. 
Furthermore, $\tilde{\bm{y}}_n = (\tilde{y}_1, ..., \tilde{y}_n)^{\top}$ is the vector of outcomes from $\tilde{\pi}$. 
Note that by definition, $\tilde{\mu}_n \in \mathcal{H}_k(\mathcal{X})$ always holds; hence, we can use $\tilde{\mu}_n$ for a good approximation of $\tilde{\pi}$. 
In this model, $\lambda > 0$ is simply a hyperparameter that is not necessarily related to the true observation noise~\cite{oliveira2019bayesian}. 

\subsection{From Gaussian process surrogate model to confidence bound for $\tilde{\pi}$}
In our setting, $\tilde{\pi}$ is an element of RKHS $\mathcal{H}_k(\mathcal{X})$ and we do not assume that it truly follows a Gaussian process. 
Nevertheless, the following result of Bogunovic et al.~\cite{bogunovic2020corruption} provides a valid confidence bound for $\tilde{\pi}$ using the same mean and variance in the surrogate GP model \eqref{eq:gaussian_process_posterior}. 

\begin{lemma}[Lemma~1 in Bogunovic et al.~\cite{bogunovic2020corruption}]
    \label{lemma:confidence_bound_for_tilde_q}
    Let $\tilde{\pi} \in \mathcal{H}_k(\mathcal{X})$ with $\|\tilde{\pi}\|_k \le B_{\tilde{\pi}}$ and $\tilde{y}_i = \tilde{\pi}(x_i) + \tilde{\varepsilon}_i(x_i)$ be the model of observations where $\tilde{\varepsilon}_i(x_i)$ is a conditionally sub-Gaussian random variable with a variance proxy of $1$ given $x_i$. 
    If we set 
    \begin{align}
        \label{eq:confidence_bound_coefficient}
        \tilde{\beta}_n^{1/2} = B_{\tilde{\pi}} + 0.5\lambda^{-1/2} \sqrt{2(\gamma_{n-1} + \ln(1/\delta))}
    \end{align}
    where 
    \begin{align}
        \label{eq:maximum_information_gain}
        \gamma_n = \max_{(x_1, ..., x_n) \subset \mathcal{X}} \frac{1}{2} \ln \det (\bm{I}_n + \lambda^{-1}\bm{K}_n)
    \end{align}
    is the maximum information gain~\cite{srinivas2010gaussian}, then the following holds with a probability of at least $1 - \delta$:
    \begin{align}
        \label{eq:confidence_bound_for_tilde_pi}
        |\tilde{\mu}_{n-1}(x) - \tilde{\pi}(x)| \le \tilde{\beta}_n^{1/2} \tilde{\sigma}_{n-1}(x), ~~~
        \forall x \in \mathcal{X}, \forall n \ge 1, 
    \end{align}
    where $\tilde{\mu}_{n-1}$ and $\tilde{\sigma}_{n-1}$ are given in the virtual GP posterior \eqref{eq:gaussian_process_posterior}. 
\end{lemma}
Then, we can derive the valid confidence bound for $\pi$ directly from Lemma~\ref{lemma:confidence_bound_for_tilde_q}. 
\begin{corollary}
    \label{corollary:confidence_bound_for_pi}
    Under the same assumptions as those in Lemma~\ref{lemma:confidence_bound_for_tilde_q}, the following holds with a probability of at least $1 - \delta$:
    \begin{align}
        \label{eq:confidence_bound_for_pi}
        |\mu_{n-1}(x) - \pi(x)| \le \tilde{\beta}_n^{1/2} \tilde{\sigma}_{n-1}(x), ~~~
        \forall x \in \mathcal{X}, \forall n \ge 1, 
    \end{align}
    where $\mu_{n-1}(x) = \tilde{\mu}_{n-1}(x) + 0.5$. 
\end{corollary}

\subsection{Convergence analysis for LSE algorithm}

\subsubsection{LSE algorithm}\label{sec: Theory ALSE algorithm}
Now, we are ready to discuss the convergence of the LSE algorithm (Algorithm~\ref{alg: ALSE algorithm}) for $\pi$. 
We begin by rewriting Algorithm~\ref{alg: ALSE algorithm} in the terminology of our setting.

\paragraph{Confidence-bound-based classification of the level sets.} 
All candidate points $x \in \mathcal{X}$ are classified for each iteration into the superlevel set $H$, the sublevel set $L$, or the undetermined set $U$. 
The classification rule is based on the confidence bound for $\pi$. 
Let $\mathrm{ucb}_{n}(x) = \mu_{n-1}(x) + \tilde{\beta}_n^{1/2} \tilde{\sigma}_{n-1}(x)$ and $\mathrm{lcb}_{n}(x) = \mu_{n-1}(x) - \tilde{\beta}_n^{1/2} \tilde{\sigma}_{n-1}(x)$ be the upper and lower confidence bounds for $\pi$, respectively (given in Corollary~\ref{corollary:confidence_bound_for_pi}). 
In each iteration, if $x \in \mathcal{X}$ satisfies $\mathrm{ucb}_{n}(x) < \theta + \xi$, then it is classified into the sublevel set $L$ because the DLT probability $\pi(x)$ is below the threshold $\theta$ with a high confidence. 
Here, $\xi$ is a non-negative real number representing the margin for classification. 
Similarly, if $x$ satisfies $\mathrm{lcb}_{n}(x) > \theta - \xi$, it is classified into the superlevel set $H$. 
Otherwise, $x$ is assigned to the undetermined set $U$ as it cannot be classified from this information alone. 
Finally, if the undetermined set becomes an empty set in an iteration, namely $U = \emptyset$, we terminate the LSE algorithm.

\paragraph{Confidence-bound-based acquisition function.} 
Next, we introduce the acquisition function commonly used in the standard LSE.
To specify the next dose, we can use the following confidence-bound-based acquisition function:
\begin{align}
    \label{eq:confidence_bound-based_acquisition}
    \alpha_n(x) = \min\{\mathrm{ucb}_{n}(x) - \theta, \theta - \mathrm{lcb}_{n}(x) \}. 
\end{align}
This type of acquisition function is referred to as {\it classification ambiguity}~\cite{Gotovos2013}; it is also known as the {\it straddle function}, especially when $\tilde{\beta}_n^{1/2}=1.96$~\cite{Bryan2005}. 
Then, we determine the next dose as $x_{n} = \argmax_{x \in \mathcal{X}} \alpha_n(x)$. 

\subsubsection{Convergence analysis for LSE algorithm}
To discuss the convergence of the LSE algorithm, we introduce the misclassification loss as follows. 
\begin{definition}
    Let $\hat{L}$ and $\hat{H}$ be the sublevel set and the superlevel set estimated by the LSE algorithm, respectively.
    For each $x \in \mathcal{X}$, we define the misclassification loss at the end of the LSE algorithm as 
    \begin{align}
        \label{eq:misclassification_loss}
        \mathcal{L}_{\theta}(x) = \begin{cases}
                                \max\{0, \pi(x) - \theta \} & \mbox{if}~~ x \in \hat{L}, \\
                                \max\{0, \theta - \pi(x)\} & \mbox{if}~~ x \in \hat{H}. 
                                \end{cases}
    \end{align}         
\end{definition}
With the aforementioned preparations, we have the following theorem for the convergence of the LSE algorithm. 
\begin{theorem}
    \label{theorem:main_theorem}
    Suppose that the regularity assumptions hold. 
    For all $\theta \in [0, 1]$, $\delta \in (0, 1)$, $\xi > 0$, if $\tilde{\beta}_n$ is set as in \eqref{eq:confidence_bound_coefficient}, then the following two statements hold. 
    \begin{enumerate}
        \item The LSE algorithm terminates after at most $N$ observations, where $N$ is the smallest positive integer satisfying 
            \begin{align}
                \label{eq:ALSE_termination}
                \frac{N}{\tilde{\beta}_{N} \gamma_{N}} \ge \frac{C_1 }{\xi^2}
            \end{align}
            where $C_1 = 8/\log(1 + \sigma^{-2})$. 
        \item With a probability of at least $1 - \delta$, the misclassification loss at the end of the LSE algorithm is less than $\xi$, i.e., 
            \begin{align}
                \label{eq:misclassification_loss_bound}
                P\left(\max_{x \in \mathcal{X}} \mathcal{L}_{\theta}(x) \le \xi \right) \ge 1 - \delta. 
            \end{align}
    \end{enumerate}
\end{theorem}
The proof of Theorem~\ref{theorem:main_theorem} is provided in Appendix~\ref{appendix: proof of theorem}. 
Note that for standard covariance functions such as the squared exponential kernel or the Mat\'ern kernel, it is known that the maximum information gain $\gamma_N$ has a sublinear order~\cite{srinivas2010gaussian, vakili2021information}. 
In particular, the order of $\gamma_N$ for the squared exponential kernel \eqref{eq: kernel} used in this study is known to be $O((\log(N))^{d+1})$, where $d$ denotes the dimension of the input. 
In our problem setting, which is the dose-finding problem, we assume that $d = 1$; thus, we have $\gamma_N = O((\log(N))^2)$. 
Consequently, $\gamma_N^2$ is also of sublinear order.
This fact, along with the statement of Theorem~\ref{theorem:main_theorem}, ensures that the LSE algorithm terminates within a finite number of trials.
From the above, we have confirmed that the theoretical convergence of LSE can be guaranteed in the setting of the dose-finding problem.

\section{Dose-finding methods based on LSE}\label{sec: Methods}
In this section, we describe the details of the proposed clinical trial design based on LSE. 
In our approach, instead of directly using Algorithm~\ref{alg: ALSE algorithm} as is, we introduce modifications to address the specific settings and constraints of phase I oncology clinical trials.
This section is structured as follows.
Section \ref{sec: model} describes how to construct the model for $\pi$ using a Gaussian process. 
Section \ref{sec: Establishing a Gaussian process prior} and \ref{sec: Dose-finding strategy} present the proposed prior specification and the proposed acquisition function, respectively, for constructing our LSE design in Section \ref{sec: trial design}.
%

\subsection{Dose-toxicity model}\label{sec: model}
Suppose that we have $J$ prespecified discrete doses denoted by $d_1<\dots<d_J$.
We assume that these doses are scaled equally on the interval $\mathcal{X} = [0,1]$ to simplify the construction of the model, although adaptation to unequally scaled doses is straightforward. For example, if we have $J=5$ doses, the scaled doses are $(d_1, d_2, d_3, d_4, d_5) = (0.00, 0.25, 0.50, 0.75, 1.00)$.
Let $Y$ be a binary DLT outcome from a particular patient, such that $Y=1$ represents the expression of DLT and $Y=0$ otherwise.
Hence, $\pi(x)$ can be given by $\pi(x) = \mathrm{Pr}(Y=1 \mid x)$.
For the data generating model on the dose-toxicity relationship, we assume the same model as that in the BO design.
Specifically, we assume the following latent Gaussian process model \cite{williams2006gaussian}:
\begin{gather}
    \mathrm{logit}(\pi(x)) = f(x),\label{eq: model} \\
    f \sim \mathrm{GP}(m,k),\label{eq: GP}
\end{gather}
where $f :\mathcal{X} \rightarrow \mathbb{R}$ is a latent function.
Note that $\pi(x)$ is a function whose value range is restricted from 0 to 1; hence, it cannot be modeled directly by a Gaussian process, and we assume a Gaussian process for the latent function $f$ on the latent space generated by the logit transformation of $\pi$.
The Gaussian process model may allow for very flexible modeling of the dose-toxicity relationship owing to its nonparametric nature, compared with model-based methods assuming parametric models for $f$.

The covariance function $k(x,x')$ has an important role in characterizing the property of $f$ sampled from the Gaussian process.
Among the various types of covariance functions, we use the squared exponential covariance function \cite{williams2006gaussian} given by
\begin{align}
    k(x,x') = \sigma_f^2 \exp\left( -\frac{(x-x')^2}{2\ell^2} \right),
    \label{eq: kernel}
\end{align}
where $\sigma_f > 0$ and $\ell > 0$ are hyperparameters.
The squared exponential covariance function is commonly used because the function $f$ generated from a Gaussian process with the squared exponential covariance function is infinitely differentiable and has a sufficiently smooth shape.
When modeling dose-response relationships, it is necessary to consider the assumption of monotonicity. 
The monotonicity of the function can be controlled by adjusting the hyperparameters of the covariance function. 
The details are discussed in Section \ref{sec: prior covariance}.

Suppose that the $i$-th patient treated with dose $x_i$ has a DLT outcome $y_i$.
Let $D_n = \{(x_i,y_i)\}_{i=1}^n$ be the observed data up to the $n$-th patient.
The likelihood function is then constructed as follows:
\begin{align}
    \mathcal{L}(D_n \mid f) &= 
    \prod_{i=1}^n \pi(x_i)^{y_i}
    \left\{1 - \pi(x_i) \right\}^{1 - y_i} \notag \\
    &= \prod_{i=1}^n \left\{\mathrm{logit}^{-1}(f(x_i))\right\}^{y_i} \left\{1-\mathrm{logit}^{-1}(f(x_i))\right\}^{1-y_i}.
    \label{eq: likelihood}
\end{align}

Given the observed data $D_n$, we derive the posterior distribution of $f$.
We generate samples from the posterior distribution of $f$ using a Markov chain Monte Carlo (MCMC) method. 
As $f$ and $\pi$ have a one-to-one correspondence, the posterior samples of $f$ are inverse logit transformed to obtain the posterior samples of $\pi$.
%

\subsection{Prior specifications in the Gaussian process model}\label{sec: Establishing a Gaussian process prior}

As described in Section \ref{sec: model}, we assume a Gaussian process prior on the latent function $f$.
We discuss how to construct the mean function $m$ and the covariance function $k$ of the Gaussian process prior by incorporating the prior information about the toxicity of the experimental drug.
The mean function $m$ should reflect our initial guess of the toxicity probability. 
However, sufficient prior information is generally unavailable when designing the phase I trial.
Therefore, we construct the mean function on the basis of limited prior information for applying our LSE design.

\subsubsection{Mean function}\label{sec: prior mean}
Here, we explain how to generate the mean function under the monotonic assumption on the dose-toxicity relationship.
Takahashi and Suzuki \cite{Takahashi2021-MTD} generated a mean function using a systematic approach to select an initial guess \cite{Lee2009-la} in the context of the CRM.
Although this approach is familiar to many clinical statisticians and can be easily implemented using the \texttt{getprior()} function in the \texttt{dfcrm} R package, it might be inappropriate because it was originally developed for the CRM. 
Therefore, we propose another approach to generate the mean function for use in our LSE design. 
Our approach is inspired by the quantile-based approach used for the prior specification in the Bayesian logistic regression model (BLRM) design \cite{neuenschwander2008critical}. 
The basic idea is to represent prior toxicity information as a quantile of the toxicity probability.

For the lowest and highest doses, we assume the following conditions:
\begin{align*}
    &\mathrm{Pr}(\pi(d_1) \geq \theta + \delta_1) = q_1,\\
    &\mathrm{Pr}(\pi(d_J) \leq \theta - \delta_1) = q_J,
\end{align*}
where $\delta_1 >0$ is a small margin and $q_1, q_J$ are pre-specified probabilities that are obtained from clinicians.
Historical information can be included to determine the values of $q_1$ and $q_J$, e.g., the results of trials of similar drugs or the results of earlier trials in other regions. 
If no historical information is available, we take a conservative value, e.g., $q_1 = q_J = 0.1$.
For the parameter $\delta_1$, we can use the proper dosing interval used in the mTPI \cite{ji2010modified} and BLRM \cite{neuenschwander2008critical} designs.
The proper dosing interval $[\theta-\delta_1,\theta+\delta_1]$ is the probability interval in which it is acceptable to include the toxicity probability at the dose selected for the MTD.

As the marginal prior distribution of $f(d_j)$ follows the normal distribution $N(m(d_j),\sigma_f^2)$ for $j=1,\dots,J$, we can derive the prior means at dose $d_1$ and $d_J$ as follows:
\begin{align}
    &m(d_1) = \mathrm{logit}(\theta + \delta_1) - z_{q_1}\tilde{\sigma}_f, \label{eq: m(d_1)} \\
    &m(d_J) = \mathrm{logit}(\theta - \delta_1) - z_{1-q_J}\tilde{\sigma}_f, \label{eq: m(d_J)}
\end{align}
where $z_q$ is the $100(1-q)$th percentile of the standard normal distribution and $\tilde{\sigma}_f$ is the mean of the prior distribution for $\sigma_f$, which is discussed in detail in Section \ref{sec: prior covariance}. 
Because we assume the prior distribution for $\sigma_f$, we replace $\sigma_f$ with its prior mean.

Assuming that $m(d_1), m(d_2), \dots, m(d_J)$ are linear, the remaining $m(d_2), m(d_3), \dots, m(d_{J-1})$ can be derived.
Alternatively, if we know the prior MTD location $\nu\in \{1,2,\dots,J\}$, we can derive the prior means as follows:
\begin{enumerate}
\renewcommand{\labelenumi}{(\roman{enumi})}
    \item If $\nu = 1$,
    modify $m(d_1)$ to $\mathrm{logit}(\theta)$ and derive the remaining prior means by linear interpolation with a line through $m(d_1)$ and $m(d_J)$.
    \item If $2 \leq \nu \leq \lfloor J/2 \rfloor$,
    set the value of $m(d_\nu)$ as $\mathrm{logit}(\theta)$ and derive the remaining prior means by linear interpolation with a line through $m(d_\nu)$ and $m(d_J)$.
    \item If $\lfloor J/2 \rfloor + 1 \leq \nu \leq J-1$,
    set the value of $m(d_\nu)$ as $\mathrm{logit}(\theta)$ and derive the remaining prior means by linear interpolation with a line through $m(d_1)$ and $m(d_\nu)$.
    \item If $\nu = J$,
    modify $m(d_J)$ to $\mathrm{logit}(\theta)$ and derive the remaining prior means by linear interpolation with a line through $m(d_1)$ and $m(d_J)$.
\end{enumerate}
For example, suppose that $J = 5, \theta=0.3$, $\delta_1 = 0.05$, $\tilde{\sigma}_f = 1.35$, and $q_1 = q_J = 0.1$.
From \eqref{eq: m(d_1)} and \eqref{eq: m(d_J)}, we get $m(d_1) = -2.35$ and $m(d_J) = 0.64$.
If $\nu = 1$, we change the value of $m(d_1)$ to $\mathrm{logit}(\theta) = -0.85$ and obtain $(m(d_1),\dots,m(d_5)) = (-0.85, -0.48, -0.11, 0.27, 0.64)$ by assuming a linear relationship from $m(d_1)$ to $m(d_5)$.
If $\nu = 2$, we set $m(d_2) = \mathrm{logit}(\theta) = -0.85$ and obtain $(m(d_1),\dots,m(d_5)) = (-1.34, -0.85, -0.35, 0.14, 0.64)$ by assuming a line drawn through $m(d_2)$ and $m(d_5)$.
The prior means obtained through this approach can be inverse logit transformed to obtain an initial guess of the toxicity probability.
Figure \ref{fig: prior} (left) shows such an initial guess on the probability scale.
Figure \ref{fig: prior} (right) shows the prior dose-toxicity relationship with the medians and equal-tailed credible intervals at each dose when the prior MTD location $\nu$ is 3.

\begin{figure}[htbp]
 \begin{minipage}{0.48\hsize}
  \begin{center}
   \includegraphics[width=65mm]{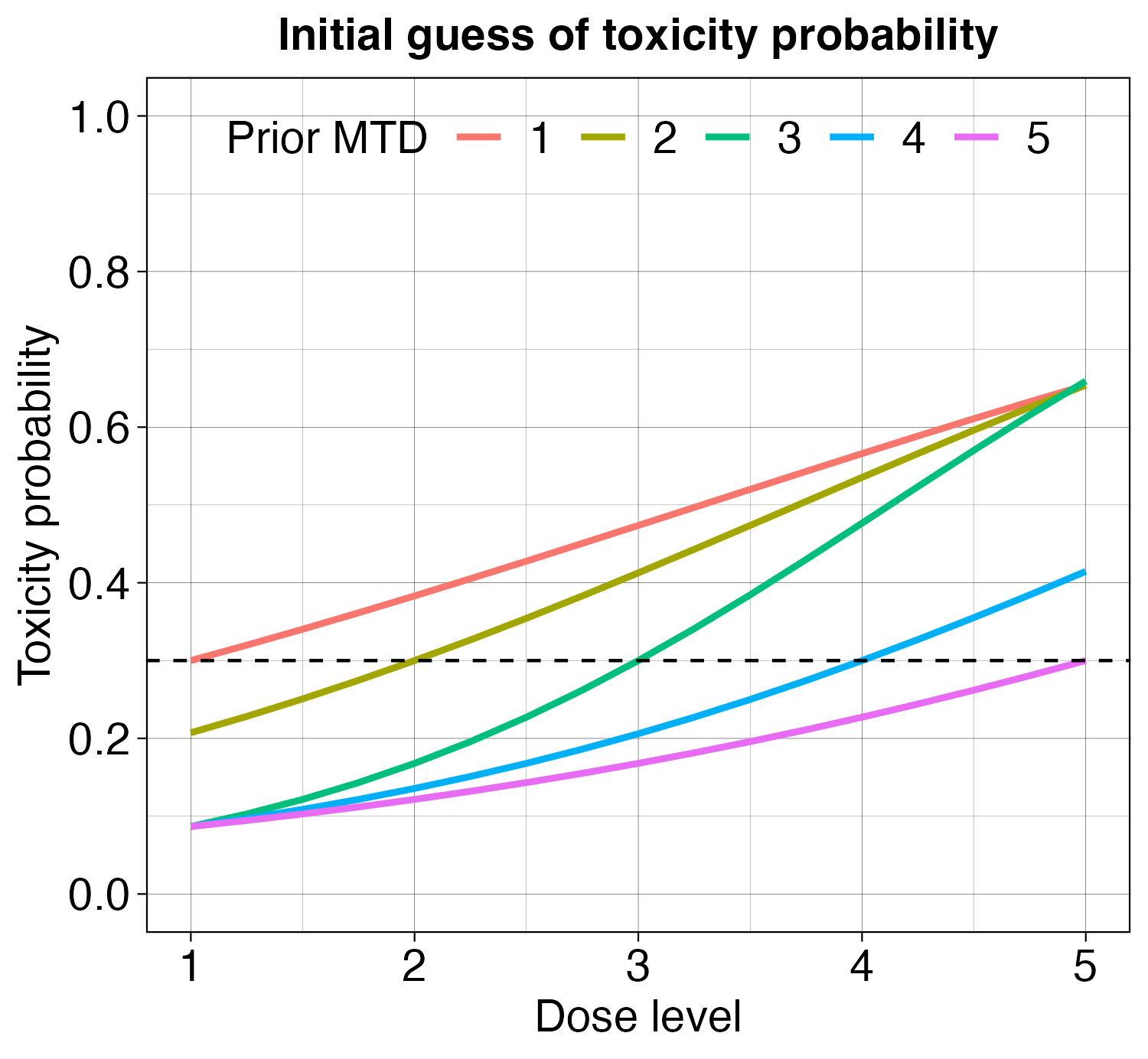}
  \end{center}
 \end{minipage}
 \begin{minipage}{0.48\hsize}
  \begin{center}
   \includegraphics[width=65mm]{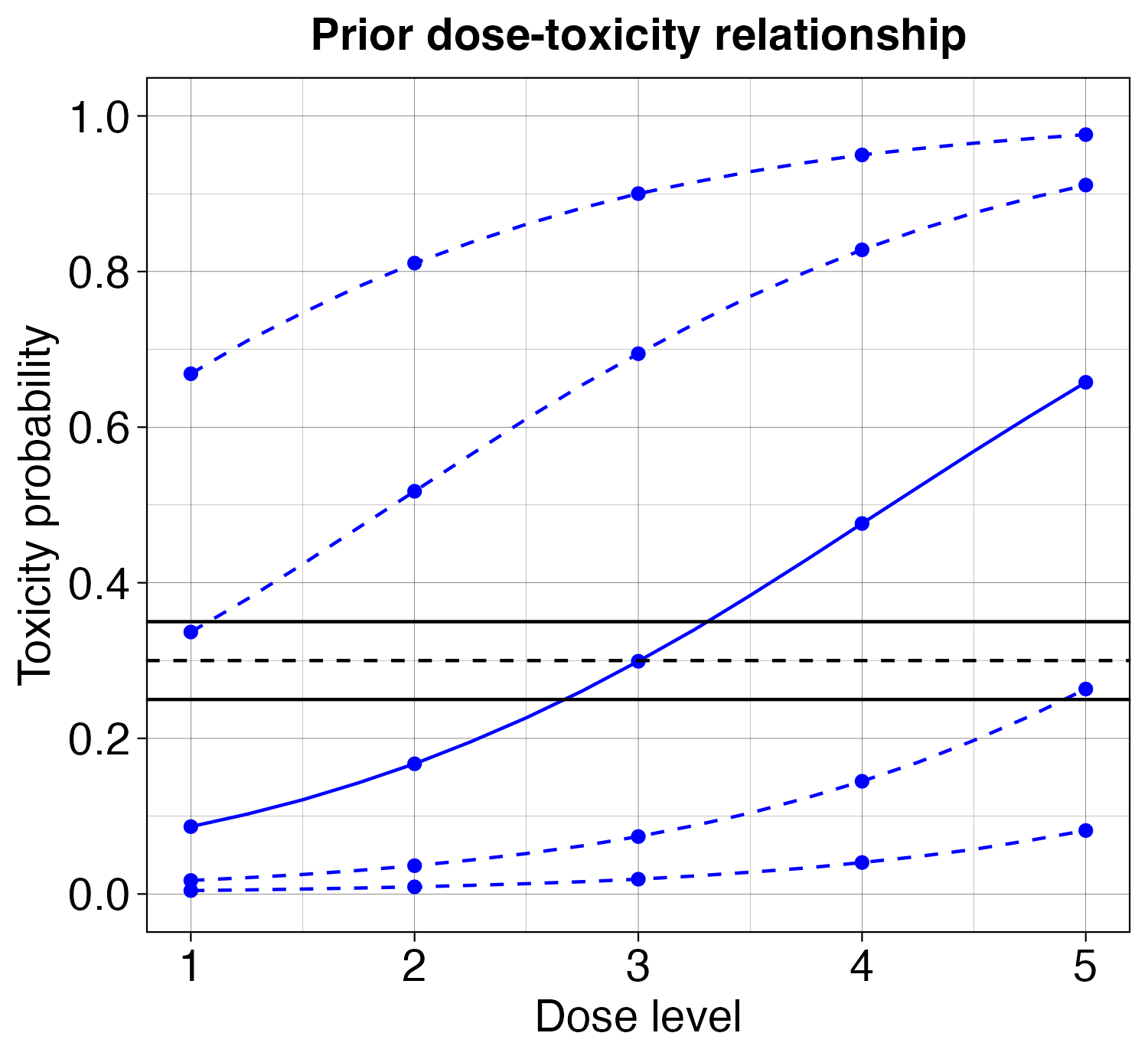}
  \end{center}
 \end{minipage}
 \caption{(left) Initial guess in each case of $\nu$ under $J = 5, ~\theta=0.3, ~\delta_1 = 0.05, ~\tilde{\sigma}_f = 1.35, ~q_1 = q_J = 0.1$. (right) Prior dose-toxicity relationship when $\nu = 3$. The solid blue line represents the median, which is equal to the initial guess. The dashed blue lines represent the 80\% and 95\% equal-tailed credible intervals. The horizontal line represents the proper dosing interval $[\theta-\delta_1, \theta + \delta_1]$.}
 \label{fig: prior}
\end{figure}

\subsubsection{Covariance function}\label{sec: prior covariance}

As described in Section \ref{sec: model}, we use the squared exponential covariance function as the covariance function of the Gaussian process. 
The following describes how to specify the hyperparameters $\sigma_f$ and $\ell$ of the squared exponential covariance function.

The hyperparameter $\ell$ controls the horizontal scale of $f$.
If the value of $\ell$ is large, the shape of $f$ and $\pi$ flattens along the mean function.
By contrast, if the value of $\ell$ is smaller than 1, which is the width of the dose range $d_J - d_1$, the shape of $f$ and $\pi$ becomes uneven and does not satisfy the monotonicity assumption of toxicity.
Therefore, following Takahashi and Suzuki \cite{Takahashi2021-MTD}, we set $\ell$ to $1$ in order to satisfy the monotonicity assumption.
This ensures that a monotonically increasing function is obtained with a high probability (see Figure \ref{fig: gp_prior_f}, \ref{fig: gp_prior_pi}).

The hyperparameter $\sigma_f$ is the marginal standard deviation of $f(x)$ at each dose $x$ and it controls the vertical scale of $f$.
As the value of $\sigma_f$ increases, the variation in the samples of $f$ and $\pi$ obtained from the Gaussian process prior increases.
We place a weakly informative prior on $\sigma_f$:
\begin{align}
    \log(\sigma_f) \sim N(\mu,\tau^2).
    \label{eq: prior for sigma_f}
\end{align}
The values of $\mu$ and $\tau$ are determined as follows.
Let $\sigma_{f1}$ and $\sigma_{f2}$ be the upper and lower bounds of the values that $\sigma_f$ is expected to take, respectively.
It may be reasonable to determine $\mu$ and $\tau$ such that the 95\% credible interval for $\sigma_f$ will be $[\sigma_{f1}, \sigma_{f2}]$.
Therefore, we can derive two conditions $\log(\sigma_{f1}) \approx \mu - 2 \tau$ and $\log(\sigma_{f2}) \approx \mu + 2 \tau$, which can be solved to obtain
\begin{align*}
\begin{split}
    \mu &= \frac{1}{2}(\log(\sigma_{f1}) + \log(\sigma_{f2})),\\
    \tau &= \frac{1}{4}(\log(\sigma_{f2}) - \log(\sigma_{f1})).
\end{split}
\end{align*}
%

\subsection{Dose-finding strategy}\label{sec: Dose-finding strategy}

\subsubsection{Criteria for determining the next dose}

As described in Section \ref{sec: level set estimation}, in the LSE framework, the point $x\in\mathcal{X}$ that maximizes the acquisition function $\alpha(x)$ is chosen as the next evaluation point. 
In the case of a dose-finding study, the dose administered to the next cohort would be the dose that maximizes the acquisition function.
However, in a dose-finding study, safety considerations are important in the dose-allocation strategy. For example, one must carefully assess whether the toxicity of the next dose to be administered is within an admissible range. 
In addition, dose skipping when escalating should be avoided from a safety perspective.
Therefore, we follow the BO design and employ an admissible dose set. The next dose to be administered will be selected from the admissible dose set.

In summary, the next dose is determined by
\begin{align}
    x_{n+1} &= \argmax_{d \in A_n } \alpha_n(d),
    \label{eq: next dose}
\end{align}
where $A_n \subseteq \{d_1,\dots,d_J\}$ is an admissible dose set at the time $n$ patients are tested.
The admissible dose set $A_n$ must satisfy all of the following conditions:
\begin{enumerate}
\renewcommand{\labelenumi}{(\roman{enumi})}
    \item All doses included in $A_n$ are less than or equal to one dose higher than the current dose $x_n$.
    \item If the lowest dose $d_1$ satisfies $\mathrm{Pr}(\pi(d_1)\geq\theta \mid D_n) \geq c_1$, i.e., the posterior probability of $d_1$ being classified as a superlevel set is greater than $c_1$, $A_n$ includes only $d_1$.
    \item All doses $d \in A_n$ must satisfy $\mathrm{Pr}(\pi(d)\geq\theta \mid D_n) \leq c_2$, i.e., the posterior probability of $d$ being classified as a superlevel set is less than $c_2$.
\end{enumerate}
Condition (i) is intended to prohibit dose skipping when escalating.
Condition (iii) is intended to ensure that doses with unacceptable toxicity are not selected as the next dose, which is similar to the dose elimination rule in the Keyboard \cite{yan2017keyboard} and BOIN \cite{liu2015bayesian} designs.
Condition (ii) is stricter than condition (iii) and therefore requires $c_1 \leq c_2$.

\subsubsection{Standard acquisition function}

First, we consider the commonly used acquisition function, i.e., the expected misclassification probability \cite{Bryan2005}, to reduce the uncertainty in classifying the tested doses into $L$ or $H$.
This is based on the posterior probability $p_n(d)$ that dose $d$ is classified as a sublevel set after obtaining the toxicity data from $n$ patients:
\begin{align*}
    p_n(d) = \mathrm{Pr}(d\in L \mid D_n) = \mathrm{Pr}(\pi(d)\leq\theta \mid D_n).
\end{align*}
Because dose $d$ will be classified as $L$ if $p_n(d) \geq 0.5$, and classified as $H$ otherwise, the expected misclassification probability is calculated as $\min (p_n(d), 1-p_n(d))$.
It is reasonable to select a dose with the highest misclassification probability as the next dose because reducing the classification uncertainty leads to accurate estimation of the level set and hence the MTD.
The misclassification probability acquisition function seems to be different from the classification ambiguity acquisition function \eqref{eq:confidence_bound-based_acquisition} considered in the theory of Section~\ref{sec: Theory}. 
However, both share a very similar concept and exhibit closely related behavior (see Appendix~\ref{appendix: Acquisition functions}). 
In other words, the convergence of our LSE algorithm is expected to approximately hold in practice even when using the misclassification probability acquisition function. 

\subsubsection{Acquisition function with overdose control}

However, it may not be appropriate to consider only the misclassification probability because it does not distinguish between overdosing and underdosing, possibly resulting in a higher risk of overdosing.
For example, for two adjacent doses $d_j$ and $d_{j+1}$, with $p_n(d_j) = 0.6$ and $p_n(d_{j+1}) = 0.4$, the value of the expected misclassification probability is $0.4$ in both cases.
However, as dose $d_{j+1}$ is expected to be more toxic than dose $d_j$ under the monotonicity assumption, it may be better to select dose $d_j$ as the next dose from the perspective of patient safety.
Therefore, we propose that the naïve misclassification probability be multiplied by a weight $p_n(d)^r$ to suppress overdose administration.
Our acquisition function is given by
\begin{align}
    \alpha_n(d) = p_n(d)^r \min (p_n(d), 1-p_n(d)),
    \label{eq: acquisition}
\end{align}
where $r \geq 0$ is a predetermined parameter.
The value of $p_n(d)$ is lower for doses with a higher expected probability of being classified in the superlevel set; therefore, the value of $\alpha_n(d)$ is also lower.
This means that a dose with higher toxicity is less likely to be chosen as the next dose.
The parameter $r$ controls how much to discount the value of the misclassification probability, and it should be calibrated before starting the trial.
When $r = 0$, $\alpha_n(d)$ corresponds to the misclassification probability itself.
A high value of $r$ will lead to strict overdose control.
The choice of the $r$ value will be discussed in Section \ref{sec: Discussion}.
Although this section has discussed the acquisition function on the basis of the expected misclassification probability, we conduct a comparison with the classification ambiguity as a baseline acquisition function in Appendix \ref{appendix: Acquisition functions}.
%

\subsection{Trial design}\label{sec: trial design}
The specification of the prior information may involve subjectivity.
Although the influence of subjectivity may diminish with data accumulation, it is preferable to avoid overdependence on the prior information, especially when the number of observations is small.
Therefore, we consider a two-stage design, following the approach of Takahashi and Suzuki \cite{Takahashi2021-MTD}.
The first stage (1, 2 below) aims to accumulate toxicity data, as little information is available on the dose-toxicity relationship.
In the second stage (3–7 below), we model the dose-toxicity curve using a Gaussian process model \eqref{eq: GP} based on the data obtained in the first stage and proceed with an adaptive dose finding based on the LSE algorithm.
The overall flow of the proposed design is as follows.

\vspace{3mm}
\noindent
First stage
\begin{enumerate}
    \item Assign the lowest dose $d_1$ to an initial cohort of patients (typically three patients).
    \item A standard design (such as the BOIN design) is implemented until a cumulative total of $n_1$ patients experience DLT or the highest dose $d_J$ is reached, where $n_1$ is a prespecified parameter.
\end{enumerate}
\noindent
Second stage
\begin{enumerate}
\setcounter{enumi}{2}
    \item Determine the prior mean function $m(x)$ on the basis of the data from the first stage.
    If we have an initial guess $a_1,\dots,a_J$ of the toxicity probability, we set the prior mean as $m(d_j) = \mathrm{logit}(a_j),~ j=1,\dots,J$.
    Otherwise, the prior mean can be derived using the quantile-based approach described in Section \ref{sec: prior mean}. We can set the prior MTD location $\nu$ to the dose recommended by the standard design for the next cohort at the end of the first stage.
    \item Generate the posterior samples over $\pi$ using the MCMC mathod.
    \item An acquisition function \eqref{eq: acquisition} is calculated using the posterior sample on $\pi$, and the dose determined by \eqref{eq: next dose} is assigned to the next cohort.
    \item Steps 4 and 5 are repeated until the total sample size $N$ is exhausted or a pre-specified stopping rule is met.
    \item The dose space $\mathcal{X}$ is partitioned into a sublevel set and a superlevel set on the basis of $p_N(d)$. If $p_N(d) \geq 0.5$, dose $d$ is classified as a sublevel set; otherwise, dose $d$ is classified as a superlevel set.
    This yields estimates $\widehat{L}$ and $\widehat{H}$ for the sublevel set and superlevel set, respectively.
    \textcolor{black}{The MTD $x^*$ is estimated by $\widehat{x}^* = \max_d \widehat{L}$.}
    The recommended dose \textcolor{black}{for the MTD} is selected as follows:
    \begin{enumerate}
    \renewcommand{\labelenumi}{(\roman{enumi})}
        \item If all candidate doses $d_1,\dots,d_J$ are classified into the superlevel set, $d_1$ is selected as the recommended dose.
        \item If all candidate doses $d_1,\dots,d_J$ are classified into the sublevel set, $d_J$ is selected as the recommended dose.
        \item If neither (a) nor (b), let $d^-$ be the highest candidate dose in $\widehat{L}$ and $d^+$ be the lowest candidate dose in $\widehat{H}$.
        For $d^-$ and $d^+$, we calculate the posterior probabilities that the toxicity probability at these doses falls in the proper dosing interval $[\theta - \delta_1, \theta + \delta_1]$:
        \begin{align*}
            u(d) := \mathrm{Pr}(\theta - \delta_1 \leq \pi(d) \leq \theta + \delta_1 \mid D_N),~~~d\in\{d^-,d^+\}.
        \end{align*}
        If $u(d^-) < u(d^+)$ and $\widehat{\pi}(d^+) \leq \theta + \delta_2$, $d^+$ is selected as the recommended dose, where $\delta_2 ~(\geq \delta_1)$ is a prespecified small value.
        Otherwise, $d^-$ is selected as the recommended dose.
    \end{enumerate}
\end{enumerate}

\section{Numerical results}\label{sec: Numerical results}
\subsection{Trial example}
We present an illustrative example of the proposed LSE design.
We considered a scenario with a dose level of $J=5$, target toxicity probability of $\theta=0.3$, cohort size of 3, sample size in the first stage of $n_1=2$, and total sample size of $N=36$.
The upper panel of Figure \ref{fig: example} shows the dose allocation path. The lowest dose was assigned to the first cohort.
The first stage ended when data from the 4th cohort was observed, with a total of 2 patients having experienced a DLT.
The lower panel of Figure \ref{fig: example} shows the posterior distribution of $\pi(x)$ after obtaining toxicity data in each cohort and the value of the acquisition function $\alpha(x)$ at each dose.
The median of the posterior distribution of $\pi(x)$ and 20 samples from this distribution are also shown.
The left panel shows the posterior distribution at the end of the first stage derived from the prior distribution with $\nu=3$.
The admissible dose set consisted of dose levels 1–4.
Dose level 3, which had the largest value of the acquisition function in the admissible dose set, was selected as the next dose.
After the 7th cohort, the width of the credible intervals for the toxicity probability at lower doses became smaller and the distribution of the dose-toxicity curves was updated downward.
The admissible dose set included all the candidate doses, and dose level 4 was selected as the next dose.
From the 7th cohort, dose level 4 was always administered until the end of the trial.
The right panel shows the posterior distribution at the end of the entire trial.
The posterior probability that each dose was classified as a sublevel set is $(p_N(d_1),\dots,p_N(d_5)) = (1.00, 1.00, 0.99, 0.68, 0.16)$; thus, we classified dose $d_1$–$d_4$ into the sublevel set and dose $d_5$ into the superlevel set.
As $u(d_4)=0.39$ and $u(d_5)=0.18$, dose $d_4$ was finally selected as the recommended dose for subsequent clinical trials.

\begin{figure}[htbp]
\begin{center}
\includegraphics[width=0.95\linewidth]{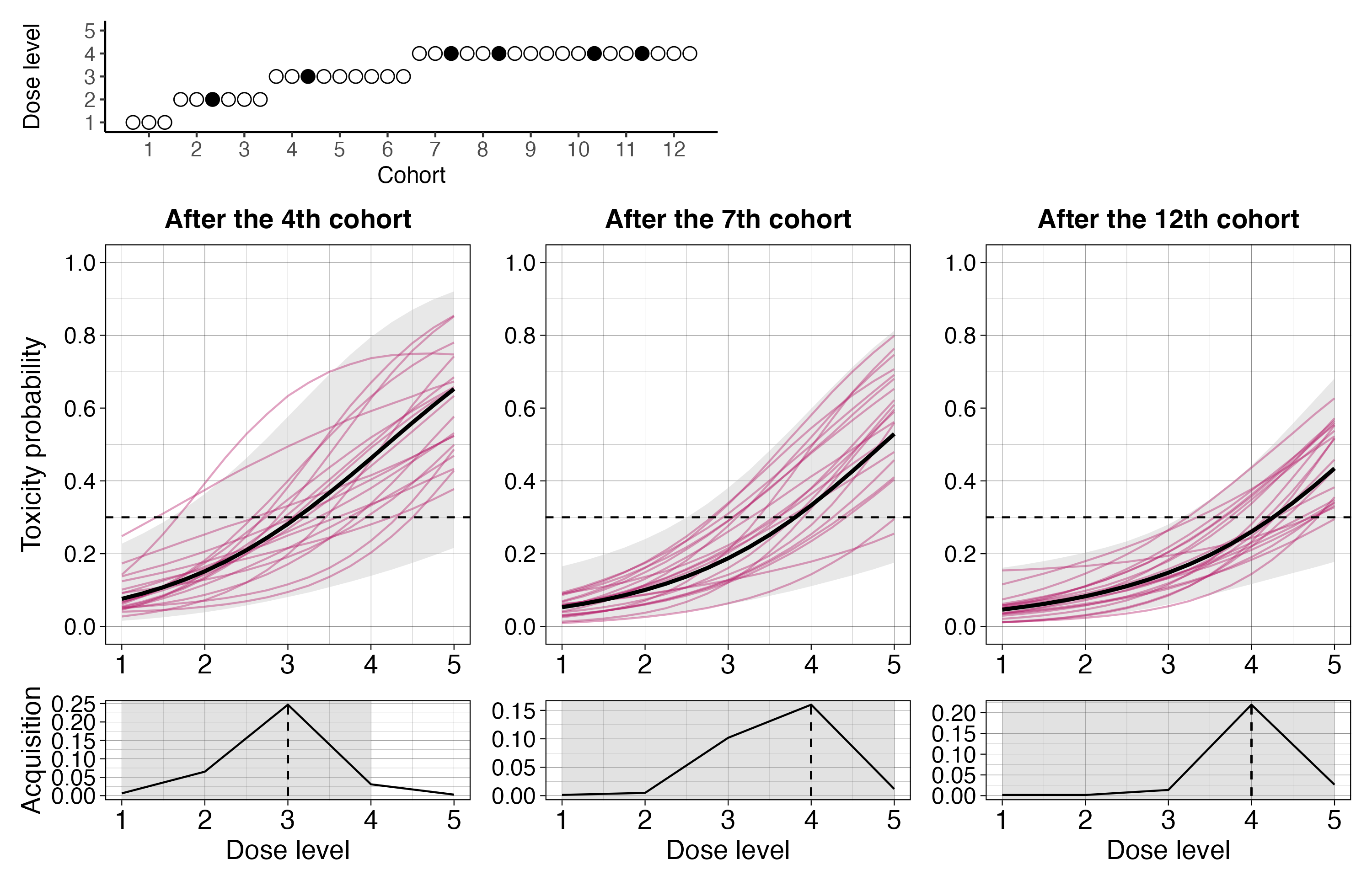}
\caption{Illustrative example of the proposed design. The upper panel shows the dose allocation path. The white circles represent patients with no DLT and the black circles represent patients with DLT. The lower panel shows the posterior distribution of $\pi(x)$ after the 4th, 7th, and 12th cohorts, with its median (bold line) and 20 random samples from them (red lines), and the value of the acquisition function $\alpha(x)$ at each dose, respectively. The shaded part of the posterior distribution of $\pi(x)$ represents the 95\% credible interval and the shaded part of the acquisition function represents the admissible dose set.}
\label{fig: example}
\end{center}
\end{figure}

\begin{table}[htbp]
\centering 
\caption{Toxicity probability scenarios. The bold values denote the MTD.}
\scalebox{0.9}{
\begin{tabular}{ccccccccccccc}
\cline{1-6} \cline{8-13}
\multicolumn{6}{c}{$\theta = 0.2$}                   &  & \multicolumn{6}{c}{$\theta = 0.3$}                   \\ \cline{1-6} \cline{8-13} 
\multirow{2}{*}{scenario}  & \multicolumn{5}{c}{dose level} &  & \multirow{2}{*}{scenario}  & \multicolumn{5}{c}{dose level} \\ \cline{2-6} \cline{9-13} 
   & 1    & 2    & 3    & 4    & 5    &  &    & 1    & 2    & 3    & 4    & 5    \\ \cline{1-6} \cline{8-13} 
1  & \textbf{0.20} & 0.26 & 0.40 & 0.45 & 0.46 &  & 11  & \textbf{0.30} & 0.36 & 0.42 & 0.45 & 0.46 \\
2  & \textbf{0.20} & 0.29 & 0.35 & 0.35 & 0.58 &  & 12  & \textbf{0.30} & 0.40 & 0.55 & 0.60 & 0.70 \\
3  & 0.10 & \textbf{0.20} & 0.25 & 0.35 & 0.40 &  & 13  & 0.08 & \textbf{0.30} & 0.38 & 0.42 & 0.52 \\
4  & 0.08 & \textbf{0.20} & 0.30 & 0.45 & 0.65 &  & 14  & 0.13 & \textbf{0.30} & 0.42 & 0.50 & 0.80 \\
5  & 0.04 & 0.06 & \textbf{0.20} & 0.32 & 0.50 &  & 15  & 0.04 & 0.07 & \textbf{0.30} & 0.35 & 0.42 \\
6  & 0.01 & 0.10 & \textbf{0.20} & 0.26 & 0.35 &  & 16  & 0.01 & 0.12 & \textbf{0.30} & 0.41 & 0.55 \\
7  & 0.05 & 0.06 & 0.07 & \textbf{0.20} & 0.31 &  & 17  & 0.06 & 0.07 & 0.12 & \textbf{0.30} & 0.40 \\
8  & 0.02 & 0.04 & 0.10 & \textbf{0.20} & 0.25 &  & 18  & 0.02 & 0.05 & 0.16 & \textbf{0.30} & 0.36 \\
9  & 0.01 & 0.02 & 0.07 & 0.08 & \textbf{0.20} &  & 19  & 0.01 & 0.02 & 0.04 & 0.06 & \textbf{0.30} \\
10 & 0.01 & 0.02 & 0.03 & 0.04 & \textbf{0.20} &  & 20 & 0.06 & 0.07 & 0.08 & 0.12 & \textbf{0.30} \\ \cline{1-6} \cline{8-13} 
\end{tabular}
}
\label{tab: scenarios}
\end{table}

\subsection{Simulations}\label{sec: Simulations}
We conducted simulations to compare the performance of the proposed LSE design with that of the CRM, BOIN, and BO designs.
See Appendix \ref{appendix: Detailed settings} for an overview of the CRM, BOIN, and BO designs as well as the specifications used in the simulations.
We considered a dose level of $J=5$ for a maximum sample size of $N=$36, with a cohort size of 3.
We considered a total of 20 scenarios with different profiles of the toxicity probability following Yan et al. \cite{yan2017keyboard}, as shown in Table \ref{tab: scenarios}.
The target probability was set as $\theta = 0.2$ for Scenarios 1–10 and $\theta = 0.3$ for Scenarios 11–20.
We performed 2000 simulations for each scenario.

For the proposed LSE design and the BO design, we adopted the two-stage approach (see Section \ref{sec: trial design} for this approach). 
We implemented the BOIN design with the default setting given by Liu and Yuan \cite{liu2015bayesian} as a common first stage design to accumulate data for applying the respective designs in the second stage. 
The first stage started with administering the lowest dose $d_1$, where the termination criterion parameter $n_1$ was set to 2.

In the second stage of the LSE design, we specified the Gaussian process prior as follows. 
For the covariance function of the Gaussian process, we set $\ell = 1$ and used a weakly informative prior $\mathrm{log}(\sigma_f) \sim N(0.20, 0.45^2)$ derived by setting $\sigma_{f1} = 0.5,~ \sigma_{f2} = 3$.
For the prior mean function, we used the quantile-based approach described in Section \ref{sec: prior mean}, assuming the general case where there was no information on the toxicity of the experimental drug before the start of the trial, and set $q_1 = q_J = 0.1$ and $\delta_1 = 0.05$.
For the acquisition function, we used the misclassification probability with overdose control as shown in \eqref{eq: acquisition}. 
In the main results, we present the results for $r = 0$ and $r = 1$. 
The results for other values of $r$ are presented in Section \ref{sec: sensitivity about gamma}.
Following Takahashi and Suzuki \cite{Takahashi2021-MTD}, we set $c_1 = 0.5$ and $c_2 = 0.9$ to construct the admissible dose set $A_n$ with $\delta_1 = 0.05$ and $\delta_2 = 0.1$.
For patient safety considerations, we stopped the trial if the condition $\mathrm{Pr}(\pi(d_1) \geq \theta \mid D_n) \geq 0.9$ was met.

We assessed the performance of each design in terms of accuracy in identifying the MTD and overdose control to ensure patient safety. 
In the following, we present the performance comparison results for each of these two aspects. 
See Appendix \ref{appendix: results} for complete numerical results.

\begin{figure}[t]
\begin{center}
\includegraphics[width=0.75\linewidth]{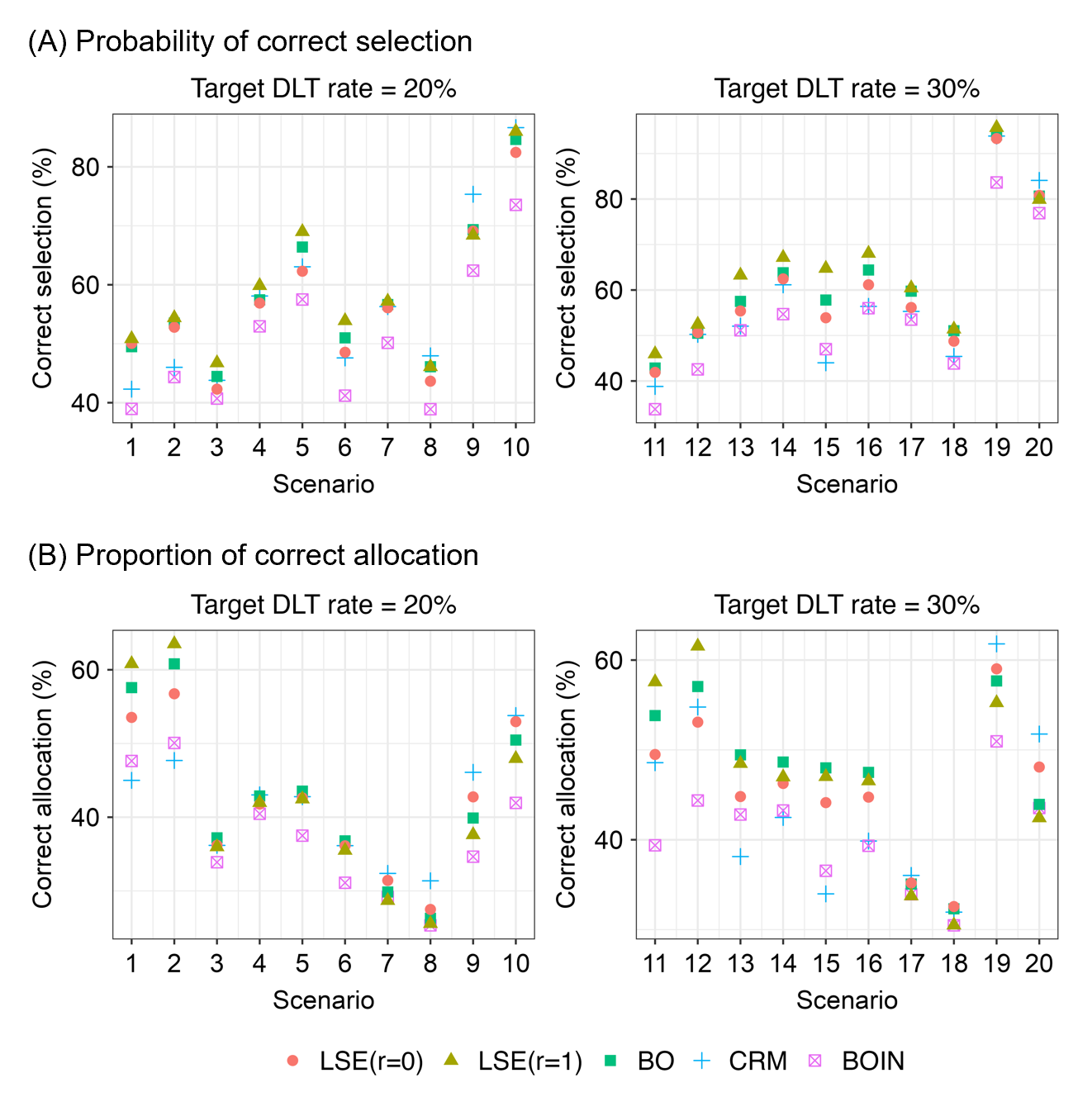}
\caption{Comparison of accuracy for identifying and allocating the MTD for each design, with PCS (A) and PCA (B) obtained from 2000 simulations. Higher values for both PCS and PCA indicate better performance.}
\label{fig: result_C}
\end{center}
\end{figure}

\subsubsection{Accuracy performance}
We evaluated two performance indices for accuracy: the probability of correct selection (PCS) and the proportion of correct allocation (PCA) of the MTD. 
PCS was calculated as the proportion of simulated trials that correctly select the target MTD as a recommended dose at the end of each simulated trial, whereas PCA was calculated as an average proportion of patients who are allocated to the target MTD within each simulated trial across 2000 simulations.

Figure \ref{fig: result_C}A shows the results of PCS for Scenarios 1–10 with a target DLT rate of 20\% (left panel) and Scenarios 11–20 with a target DLT rate of 30\% (right panel).
The LSE design with $r=1$ had the highest PCS values for all scenarios except Scenarios 8, 9, 10, and 20.
The LSE design with $r=0$ tended to show slightly lower PCS values than that with $r=1$; however, it showed PCS values comparable to or higher than the CRM design in most scenarios except for Scenarios 8, 9, 10, and 20.
With respect to the other designs, the BOIN design exhibited the worst performance in most scenarios except Scenario 15.
The BO design outperformed the CRM and BOIN designs in many scenarios, while it was slightly inferior to the LSE design with $r=1$.
The CRM design showed the best PCS for Scenarios 8, 9, 10, and 20, and the worst values for Scenario 15.
This indicates that the performance of the CRM depends on the validity of the assumptions of the parametric dose-toxicity model.
Meanwhile, the LSE and BO designs based on the non-parametric Gaussian process model generally provided higher PCS values than the other methods across all scenarios with various dose-toxicity profiles.

In the case of PCA, the relative performance among the designs appeared to be more scenario-dependent, as shown in Figure \ref{fig: result_C}B.
When the MTD was the lowest dose (Scenarios 1, 2, 11, and 12), the LSE design with $r=1$ showed the highest PCA values, followed by the BO design with a difference of 3\%–4\%.
The LSE design with $r=0$ showed lower PCA values than that with $r=1$.
Meanwhile, the CRM and BOIN designs showed lower PCA in these scenarios.
These results can be explained by a greater opportunity for administering overdoses to patients by CRM (see Figure \ref{fig: result_O}B) and a higher percentage of early stopping of the trial exceeding 30\% in these scenarios by BOIN (see Figure \ref{fig: Early stop}).
When the MTD was dose level 2 or 3 (Scenarios 3–6,13–16), the BO design showed better performance.
In particular, in scenarios with a target DLT rate of 20\%, LSE, BO, and CRM performed comparably, whereas BOIN was inferior to these designs.
In Scenarios with a target DLT rate of 30\%, LSE with $r=1$ was slightly inferior to BO, but superior to BOIN and CRM.
Finally, when the MTD was higher dose levels 4 or 5 (Scenarios 7–10,17–20), CRM showed higher PCA than the other designs.
The LSE design with $r=0$ showed higher PCA than that with $r=1$.

\begin{figure}[t]
\begin{center}
\includegraphics[width=0.75\linewidth]{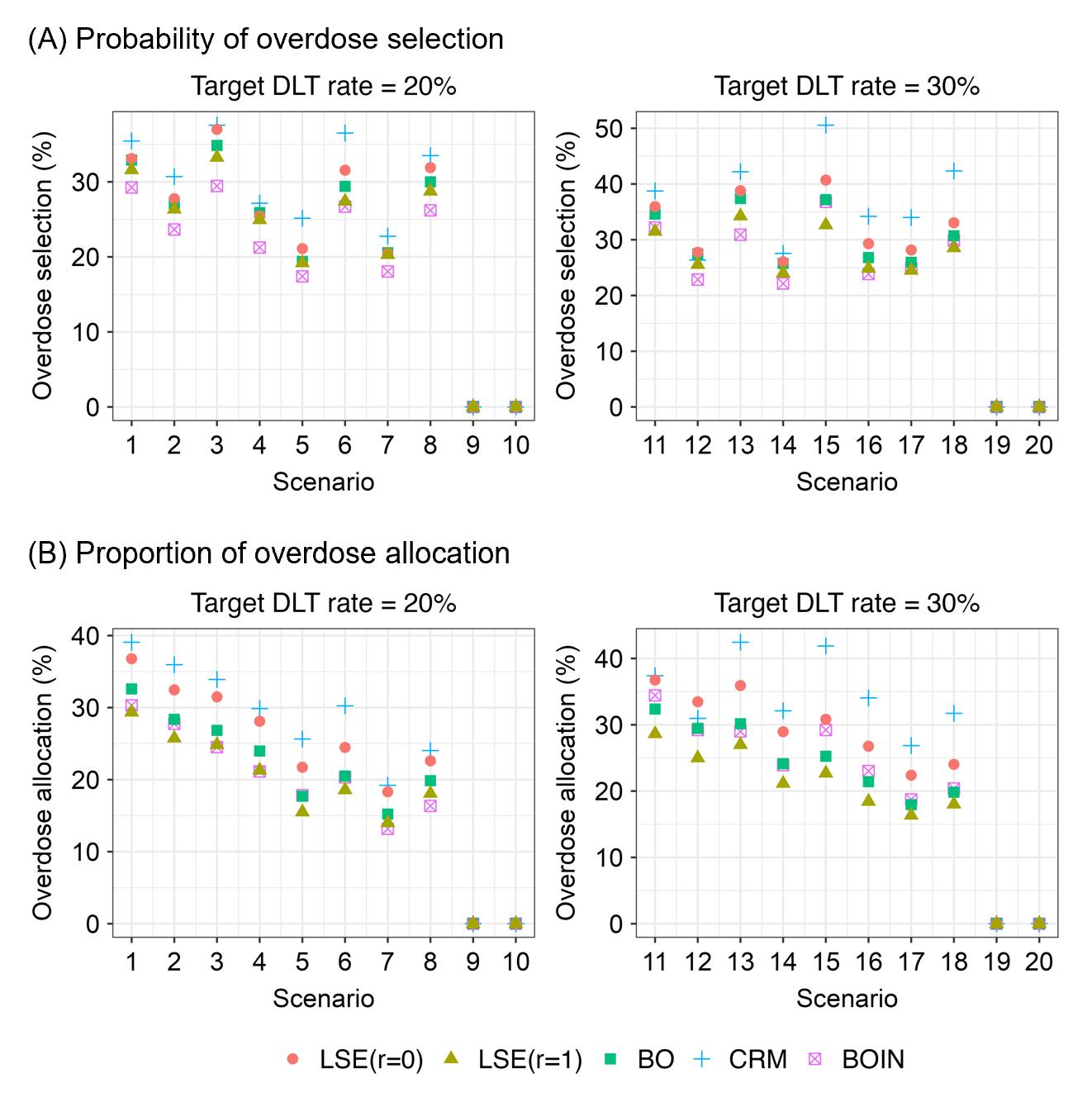}
\caption{Comparison of the risk of overdosing for each design with POS (A) and POA (B) obtained from 2000 simulations. Lower values for both POS and POA indicate better performance.}
\label{fig: result_O}
\end{center}
\end{figure}

\subsubsection{Safety performance}
We evaluated two performance indices for overdose related to patient safety: the probability of overdose selection (POS) and the proportion of overdose allocation (POA).
POS was calculated as the proportion of simulated trials that select a higher dose than the target MTD as a recommended dose at the end of each simulated trial, whereas POA was calculated as an average proportion of patients who are allocated to doses above the MTD within each simulated trial across 2000 simulations.

Figures \ref{fig: result_O}A and \ref{fig: result_O}B show the results of POS and POA, respectively.
First, we note that, in Scenarios 9, 10, 19, and 20, both POS and POA were zero because the target MTD was the highest dose in these scenarios.
According to Figure \ref{fig: result_O}A, the BOIN design provided the best POS values in most scenarios, while the LSE design with $r=1$ provided the next lowest POS values in many scenarios. 
Although BOIN generally outperformed LSE ($r=1$), their performances were comparable, especially in Scenarios 6, 11, 14, and 16–18.
The LSE design with $r=1$ tended to yield a lower POS than that with $r=0$, although the difference was marginal in some scenarios.
The CRM provided the highest POS in all scenarios except Scenario 12, indicating that its overdose control performance is poor compared to the other design methods.

In terms of POA, the LSE design with $r=1$ demonstrated good overdose control performance, achieving the lowest POA values in all scenarios except Scenarios 3, 4, 7, and 8 (see Figure \ref{fig: result_O}B). 
Notably, in Scenarios 1–10, the BOIN design exhibited performance comparable to that of the LSE design with $r=1$, also achieving low POA values.
In particular, the LSE design with $r=1$ provided safer dose allocation than the BO and CRM designs throughout all the scenarios. When $r$ was set to 0 instead of 1 in the LSE design, the POA values increased by more than 5\% in most scenarios, indicating a notable reduction in overdose control.
The CRM consistently had the highest risk of overdosing, with a difference in POA values of up to 10\% compared with the other designs. 
Consequently, CRM also tended to result in a higher proportion of patients experiencing DLTs throughout the trial (see Figure \ref{fig: DLT_p}).

\begin{figure}[htbp]
\begin{center}
\includegraphics[width=0.95\linewidth]{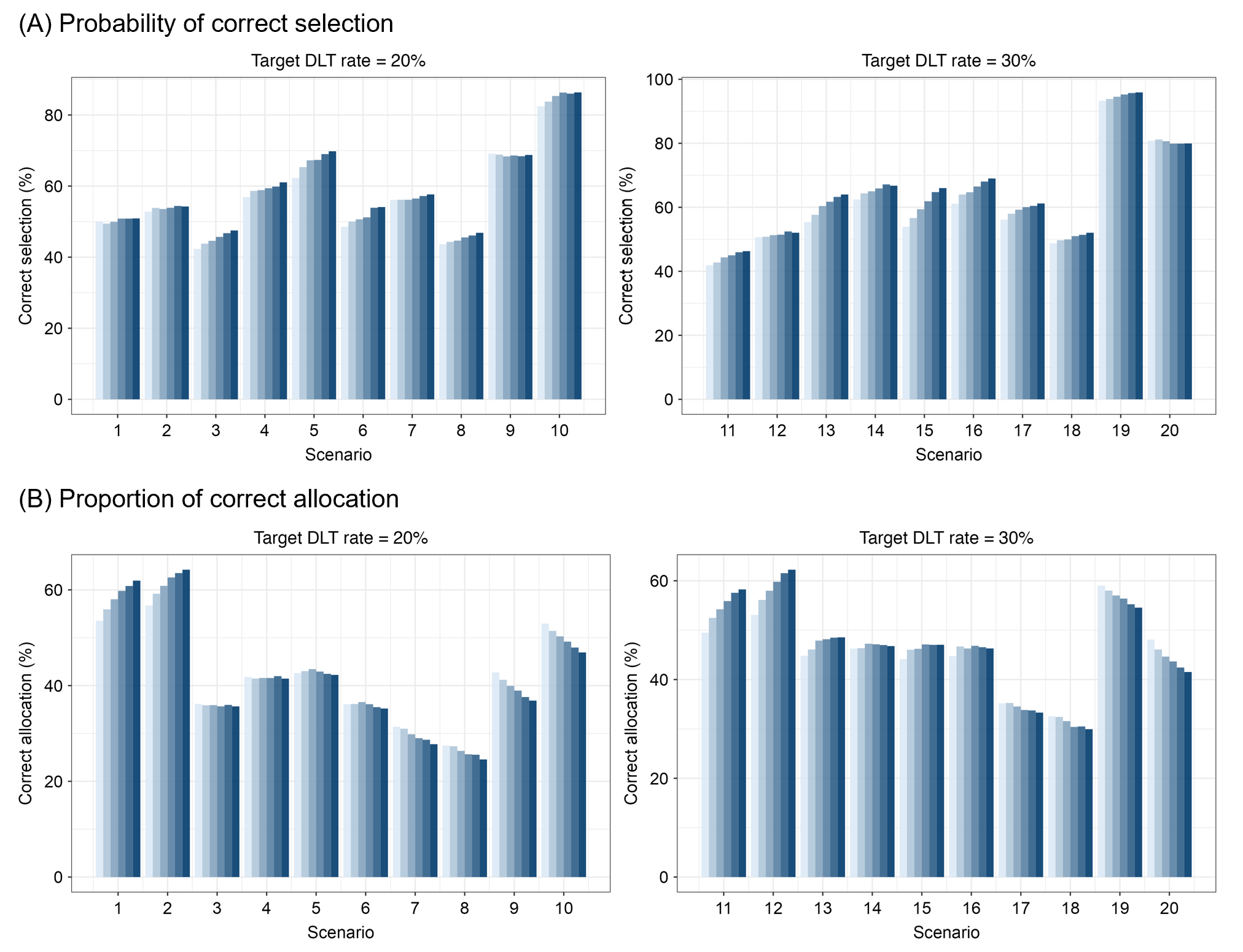}
\caption{Sensitivity to the value of $r$. For each scenario, from left to right, the values of $r$ correspond to 0, 0.25, 0.50, 0.75, 1.00, and 1.25. The upper panel shows PCS (A) and the lower panel shows PCA (B).}
\label{fig: result_C_gamma}
\end{center}
\end{figure}

\subsubsection{Sensitivity to the value of the parameter $r$}\label{sec: sensitivity about gamma}
For the proposed LSE design, we also evaluated how the value of the tuning parameter $r$ in the acquisition function \eqref{eq: acquisition} affects the performance indices.
Figures \ref{fig: result_C_gamma} and \ref{fig: result_O_gamma} show the results of the respective performance indices for $r$ values of 0, 0.25, 0.50, 0.75, 1.00, and 1.25.
Here, all the specifications in the LSE design, except the value of $r$, were the same as those in the previous simulations.
In some scenarios, the PCS value was constant over various values of $r$, while in other scenarios, the PCS value increased with $r$; however, the degree of increase varied from scenario to scenario.
In scenarios where the MTD was the lowest dose, PCA increased with $r$.
In scenarios where the MTD was dose level 4 or 5, PCA decreased with $r$.
In general, larger values of $r$ were associated with conservative dose allocations, and the aforementioned results can be explained by this property.
Meanwhile, it is interesting to note that, when the MTD was an intermediate dose between the candidate doses, the value of $r$ had little effect on PCA.
Finally, with respect to the impact on overdose control shown in Figure \ref{fig: result_O_gamma}, in nearly all the scenarios, both POS and POA generally decreased for larger values of $r$; in particular, a significant decreasing trend was observed for POA.
Again, this implies that a larger $r$ results in a more conservative dose allocation.

\begin{figure}[htbp]
\begin{center}
\includegraphics[width=0.95\linewidth]{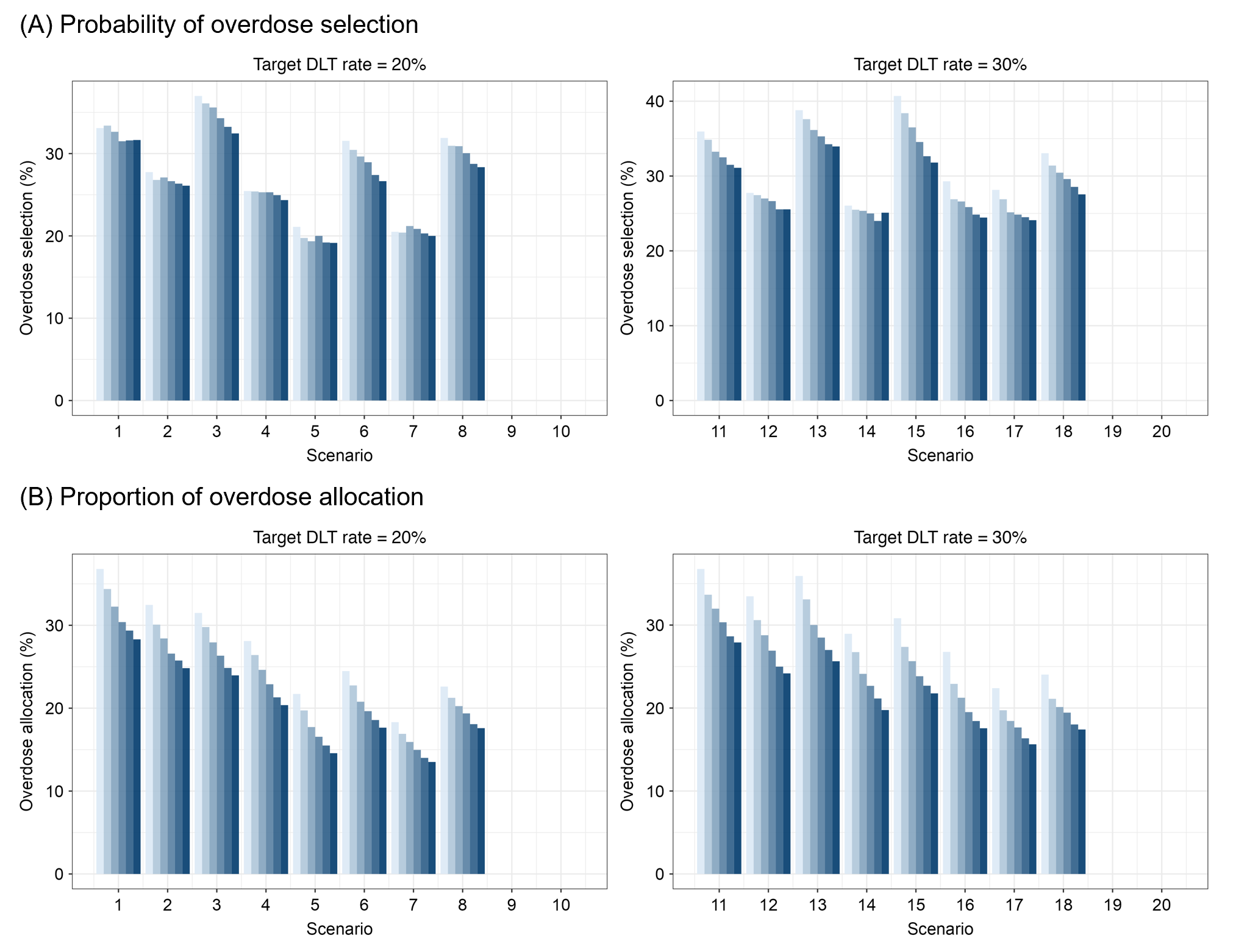}
\caption{Sensitivity to the value of $r$. For each scenario, from left to right, the values of $r$ correspond to 0, 0.25, 0.50, 0.75, 1.00, and 1.25. The upper panel shows POS (A) and the lower panel shows POA (B).}
\label{fig: result_O_gamma}
\end{center}
\end{figure}

\section{Discussion}\label{sec: Discussion}
To the best of our knowledge, this study is the first to attempt to address the dose-finding problem in phase I cancer clinical trials as a problem of LSE \cite{Gotovos2013}. %
According to a formulation of the LSE design for the dose-finding problem, we proposed an acquisition function to efficiently select the next dose while controlling the risk of overdose administration as well as methods for specifying the priors on the basis of limited toxicity information. 
The Bayesian nonparametric model for the dose-toxicity curve using the Gaussian process prior is flexible in terms of eliminating the risk of model misspecification and efficient in terms of sharing toxicity information across multiple doses, thereby addressing the drawbacks of model-based and model-assisted designs.
The proposed LSE design performed as well as or better than the CRM, BOIN, and BO designs in terms of both MTD selection accuracy and overdose prevention in many scenarios with various dose-toxicity profiles. 

Another advantage of the proposed LSE design is that it allows the inference of the toxicity probabilities even for a particular dose that is not included in the tested doses $d_1,\dots,d_J$.
Accordingly, we can estimate the MTD $x^*$ that may fall between the tested doses, so that the estimation of the MTD is built into the LSE design, in contrast to many model-assisted designs.
This advantage is due to the fact that the Gaussian process model for the dose-toxicity curve can handle any continuous input $x\in\mathcal{X}$.
This estimation of the MTD provides useful information to determine the recommended dose in subsequent clinical trials.
For example, accurate MTD estimation can contribute toward the detection of an incorrect dose specification during the course of phase I clinical trials.
In some cases, the tested doses do not include those whose toxicity probabilities are not close to the target probability, which is referred to as a misspecification of the tested doses \cite{Hu2013-nw,Chu2016-vz}.
If the entire dose space is classified as $L$ or $H$, this may imply a misspecification of the tested doses.
Furthermore, when $p_n(d_j)$ is too large and $p_n(d_{j+1})$ is too small for two adjacent doses, it also implies a misspecification, i.e., there might be a more appropriate dose close to the MTD between these two doses.
Accurate estimation of the MTD may also be particularly useful if the dose setting can be modified on the basis of the estimation of the MTD $x^*$ in a trial with dose expansion cohorts involving more than one dose to further investigate the toxicity profile or a phase II trial to evaluate the treatment efficacy after phase I dose exploration \cite{iasonos2013design}.

In contrast to the CRM and BOIN designs whose dose-selection strategy is based on point estimation, the proposed design determines the next dose on the basis of the uncertainty in the posterior distribution of $\pi$, similar to the BO design, possibly resulting in a more efficient dose-selection rule.
The CRM design summarizes information on the posterior distribution of toxicity probabilities across doses into a point estimate of the posterior mean.
Accordingly, there may be a loss of information and an increase in the risk of selecting too toxic or subtherapeutic doses, as posterior distributions showing different toxicities may have the same posterior mean.
Although it might be possible for the CRM to calculate a similar acquisition function on the basis of the uncertainty of the posterior distribution and use it to develop a dose-selection strategy, it is not clear whether flexible dose selection is possible because the parametric model in the CRM is less flexible than a non-parametric Gaussian process model.

For the proposed LSE design, a possible reason for a lower risk of allocating too toxic doses compared to the other designs is the introduction of the weight $p(d)^r$ in our acquisition function \eqref{eq: acquisition} to avoid overdosing.
Such a weight is often introduced to avoid searching in dangerous areas in LSE applications \cite{snoek2013bayesian,gardner2014bayesian}.
However, we must specify the value of the parameter $r$ before starting the trial.
The sensitivity analysis of $r$ in Section \ref{sec: sensitivity about gamma} showed that increasing the value of $r$ reduces the risk of overdosing and improves the accuracy of MTD selection in some scenarios.
Furthermore, as the value of $r$ increases, the proportion of patients allocated to the MTD decreases in scenarios where the MTD is among higher doses.
Considering the trade-off between accuracy and safety in dose allocation, we recommend $r=1$ for a $\theta$ range of 0.2– 0.3.
The value of $r$ can be calibrated to maximize some appropriate criterion under various possible toxicity scenarios before starting the trial.
As a possible criterion, we can consider a composite index, such as $w_1 \mathrm{PCS} + w_2 \mathrm{PCA} - w_3 \mathrm{POS} - w_4 \mathrm{POA}$, where $w_1,w_2,w_3,w_4 \geq 0$ are predetermined parameters that adjust the balance of each performance index.
In practice, simulations using the MCMC method are time-consuming; hence, it might be feasible to use some approximation method such as the Laplace approximation or the expectation propagation to derive the posterior distribution.
See \cite{williams2006gaussian} for details on these approximation methods.

Although we have employed the acquisition function based on the misclassification probability, other acquisition functions can also be considered in our design.
One is a classification ambiguity acquisition function \cite{Gotovos2013} that quantifies the uncertainty in classification on the basis of the credible interval of the posterior distribution.
We observed no major differences in performance with respect to dose identification by acquisition functions based on the misclassification probability and those based on the classification ambiguity (see Appendix \ref{appendix: Acquisition functions}).
As a more sophisticated acquisition function, a look-ahead acquisition function has been proposed \cite{letham2022look}; however, it may not perform well in our situation with a simple monotonic function $\pi$ for a one-dimensional input.
Indeed, our exploratory simulation experiments showed that the look-ahead acquisition function tends to select a dose that has not been previously administered as the next dose, resulting in an increased risk of overdosing (data not shown).
Further study is warranted to investigate how the performance of the design is affected by the choice of the acquisition function.

Although we have considered the problem of MTD identification in the context of phase I clinical trials of single treatments, the proposed LSE design can be extended to more complex designs.
The method can be easily extended to drug combination trial designs \cite{Mandrekar2014-xg} by considering a Gaussian process model with two-dimensional inputs to search for the level set of doses whose toxicity probabilities are equal to $\theta$, i.e., the MTD contour.
Moreover, recent experimental therapies, such as molecular targeted agents, immunotherapy, and radiation therapies, require evaluation of the efficacy and immune response as well as toxicity.
In such cases, the objective of the trial is to estimate an optimal biological dose, not the MTD \cite{Postel-Vinay2016-bt}.
In the future, we will consider extending the proposed design to such complex clinical trial designs.
%

\section*{Acknowledgments}

This work was supported by Grants-in-Aid from the Japan Society for the Promotion of Science (JSPS) for Scientific Research (KAKENHI grant nos. JP20K19871 and JP24K20836 to K.M., JP23K16943 to Y.I.).

\bibliographystyle{unsrt}
\bibliography{references}

\clearpage
\appendix
\renewcommand{\theequation}{\thesection\arabic{equation}}
\renewcommand{\thefigure}{\thesection\arabic{figure}}
\renewcommand{\thetable}{\thesection\arabic{table}}
\setcounter{equation}{0} 
\setcounter{figure}{0} 
\setcounter{table}{0} 

\section{Proof of Theorem~\ref{theorem:main_theorem}}\label{appendix: proof of theorem}

\begin{lemma}
    \label{lemma:acquisition_bound}
    For any $n \ge 1$, the following holds:
    \begin{align}
        \label{eq:acquisition_bound}
        \alpha_n(x_n) \le \tilde{\beta}_n^{1/2} \tilde{\sigma}_{n-1}(x_n)
    \end{align}
\end{lemma}
\begin{proof}
    \begin{align*}
        \alpha_n(x_n) &= \min\{\mathrm{ucb}_{n}(x_n) - \theta, \theta - \mathrm{lcb}_{n}(x_n) \} \\ 
        &= \min\{\mu_{n-1}(x_n) + \tilde{\beta}_n^{1/2} \tilde{\sigma}_{n-1}(x_n) - \theta, \theta - \mu_{n-1}(x_n) + \tilde{\beta}_n^{1/2} \tilde{\sigma}_{n-1}(x_n) \} \\ 
        &= \tilde{\beta}_n^{1/2}\tilde{\sigma}_{n-1}(x_n) - |\mu_{n-1}(x_n)  - \theta| \\ 
        &\le \tilde{\beta}_n^{1/2}\tilde{\sigma}_{n-1}(x_n)
    \end{align*}
\end{proof}

\begin{lemma}
    \label{lemma:acquisition_bound2}
    While running the LSE algorithm with $\tilde{\beta}_n$ as in \eqref{eq:confidence_bound_coefficient}, for all $n \in \mathbb{N}$, there exists $n^\prime \le n$ such that 
    \begin{align}
        \alpha_{n^\prime}(x_{n^\prime}) \le \sqrt{\frac{C_1 \tilde{\beta}_n \gamma_n}{n}}
    \end{align}
    where $C_1 = \frac{1}{\log(1 + \lambda^{-2})}$. 
\end{lemma}
\begin{proof}
    From Lemma~\ref{lemma:acquisition_bound}, for any $i \ge 1$, we have 
    \begin{align*}
        \alpha_i^2(x_i) 
        &\le \tilde{\beta}_i \tilde{\sigma}_{i-1}^2(x_i) \\ 
        &= \tilde{\beta}_i \lambda^2 \left(\lambda^{-2} \tilde{\sigma}_{i-1}^2(x_i) \right) \\ 
        &\le \tilde{\beta}_i \lambda^2 \frac{\lambda^{-2}}{\ln(1 + \lambda^{-2})} \ln \left(1 + \lambda^{-2} \tilde{\sigma}_{i-1}^2(x_i) \right). 
    \end{align*}
    As $\tilde{\beta}_n$ is monotonically increasing with respect to $n$, the following holds for any $n \ge 1$:
    \begin{align*}
        \sum_{i=1}^n \alpha_i^2(x) 
        &\le \tilde{\beta}_n \frac{1}{\log(1 + \lambda^{-2})} 
        \underbrace{\sum_{i=1}^n \ln \left(1 + \lambda^{-2}\tilde{\sigma}_{i-1}^2(x_i)  \right)}_{=\ln\det( \bm{I}_n + \lambda^{-1} \bm{K}_n)} \\ 
        &\le \tilde{\beta}_n \frac{1}{\log(1 + \lambda^{-2})} 
        \underbrace{\max_{(x_1, ..., x_n) \in \mathcal{X}} \ln \det \left( \bm{I}_n + \lambda^{-1} \bm{K}_n \right)}_{=\gamma_n}. 
    \end{align*}
    Let $n^* = \argmin_{n^\prime \le n} \alpha_{n^\prime}^2(x_{n^\prime})$; then, it is obvious that
    \begin{align*}
        \sum_{i=1}^n \alpha_i^2(x_i) \ge n \alpha_{n^*}^2(x_{n^*}). 
    \end{align*}
    Hence, if we set $C_1 = \frac{1}{\log(1 + \lambda^{-2})}$, 
    \begin{align*}
        \alpha_{n^*}^2(x_{n^*}) \le \frac{C_1\tilde{\beta}_n\gamma_n}{n}
    \end{align*}
    holds and we obtain the claim of the lemma. 
\end{proof}

\begin{lemma}
    \label{lemma:undetermined_set}
    While running the LSE algorithm, if there exists $n \ge 1$ such that $\alpha_n(x_n) \le \xi$, then the undetermined set at $n+1$ is empty, i.e., $U_{n+1} = \emptyset$. 
\end{lemma}
\begin{proof}
    By the definition of $\alpha_n(x)$, it follows that 
    \begin{align*}
        \xi 
        &\ge \alpha_n(x_n) \\
        &= \max_{x \in \mathcal{X}} \min \{\mathrm{ucb}_n(x) - \theta, \theta - \mathrm{lcb}_n(x)\} \\ 
        &\ge \min \{\mathrm{ucb}_n(x) - \theta, \theta - \mathrm{lcb}_n(x)\}, ~~~ \forall x \in \mathcal{X}
    \end{align*}
    This implies that for any $x \in \mathcal{X}$, either $\mathrm{ucb}_n(x) - \theta \le \xi$ or $\theta - \mathrm{lcb}_n(x) \le \xi$ is always true. 
    The former means that $x$ is an element of the superlevelset $H$, while the latter means that $x$ is an element of the sublevelset $L$. 
    Thus, if $\alpha_n(x_n) \le \xi$ holds, all $x \in \mathcal{X}$ are classified into $H$ or $L$ and the undetermined set is an empty set.  
\end{proof}
Then we have the following statement from Lemmas~\ref{lemma:acquisition_bound2} and \ref{lemma:undetermined_set}.
\begin{corollary}
    \label{corollary:LSE_termination}
    The LSE algorithm terminates after at most $N^\prime$ observations, where $N^\prime$ is the smallest positive integer satisfying 
    \begin{align*}
        N^\prime \ge \frac{C_1 \tilde{\beta}_{N^\prime} \gamma_{N^\prime}}{\xi^2}
    \end{align*}
\end{corollary}
Theorem 1 is the direct consequence of Lemma~\ref{lemma:undetermined_set} and Corollary~\ref{corollary:LSE_termination}.

\clearpage
\section{Gaussian process prior}\label{appendix: GP prior}

\begin{figure}[htbp]
\begin{center}
\includegraphics[width=0.85\linewidth]{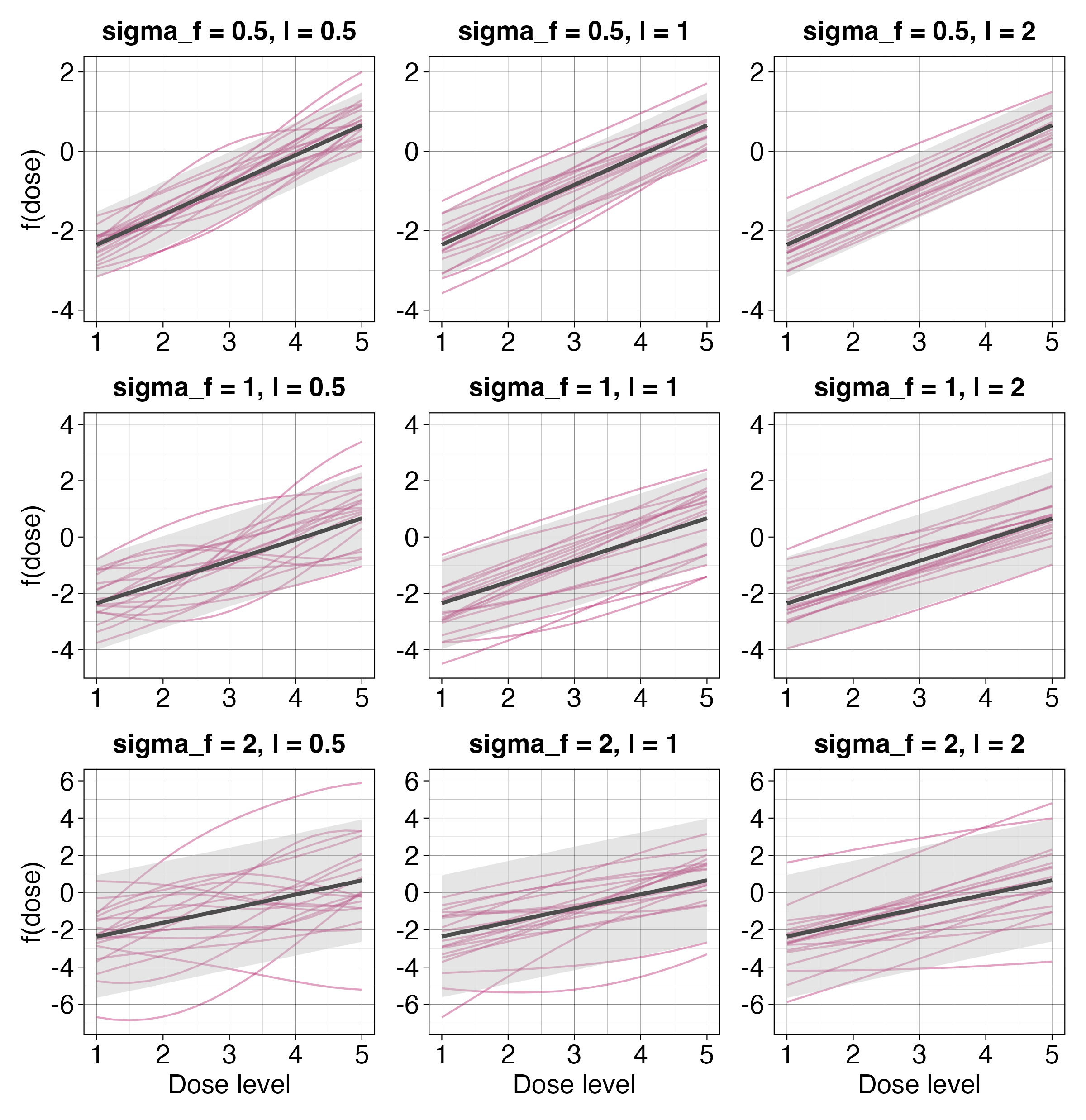}
\caption{Gaussian process prior for the latent function $f(x)$ with the hyperparameters $\sigma_f$ and $\ell$ set to 0.5, 1, and 2 respectively. The bold line represents the mean and the red line represents the 20 random samples from the Gaussian process prior. The shaded area represents the 90\% credible interval.}
\label{fig: gp_prior_f}
\end{center}
\end{figure}

\begin{figure}[htbp]
\begin{center}
\includegraphics[width=0.85\linewidth]{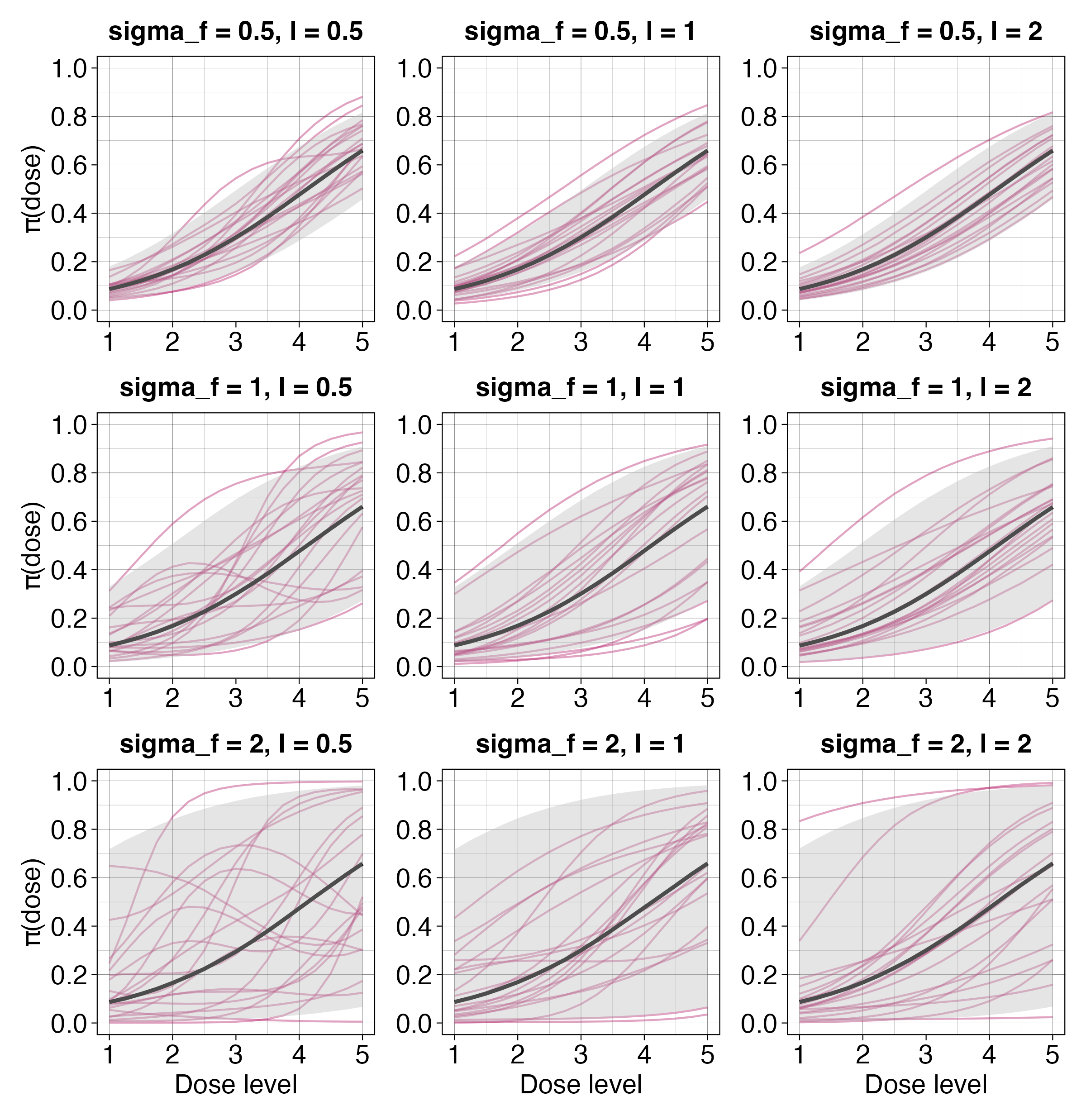}
\caption{Prior for the toxicity probability $\pi(x)$ with the hyperparameters $\sigma_f$ and $\ell$ set to 0.5, 1 and 2 respectively. The bold line represents the median and the red line represents the 20 random samples from the Gaussian process prior, obtained by the inverse logit transformation of the samples for $f$ in Figure \ref{fig: gp_prior_f}. The shaded area represents the 90\% credible interval.}
\label{fig: gp_prior_pi}
\end{center}
\end{figure}

\clearpage
\section{Overview of existing designs and their detailed settings in the simulations}\label{appendix: Detailed settings}
\subsection{CRM}
In the CRM, we consider the power model $\pi(d_j) = a_j^{\exp(\beta)}$, for $j=1,\dots,J$, where $\beta$ is a hyperparameter and $a_1 < a_2 < \dots < a_J$ are initial guesses of the toxicity probability at each dose, referred to as the skeleton.
We select the skeleton using the systematic approach of Lee and Cheung \cite{Lee2009-la}.
We set the prior MTD to dose 3, which is the intermediate dose, and the half-width of the indifference interval to 0.05.
Consequently, the skeleton is selected as $(0.05, 0.11, 0.20, 0.31, 0.42)$ for $\theta=0.2$ and $(0.12, 0.20, 0.30, 0.40, 0.50)$ for $\theta=0.3$.
We assume a prior $\beta \sim N(0,2)$.

The next dose is basically the dose with the posterior mean of the toxicity probability closest to $\theta$.
However, for safety reasons, we do not allow dose skipping when increasing or decreasing the dose in the simulation.
In addition, we stop the trial when the condition $\mathrm{Pr}(\pi(d_1) \geq \theta \mid D_n) \geq 0.9$ is met.
%

\subsection{BOIN}
In the BOIN design, the next dose is determined on the basis of the boundaries $\lambda_\mathrm{e}, \lambda_\mathrm{d}$ for dose escalation and deceleration and the estimated toxicity probability $\hat{\pi}(d_j)$ at the current dose $j$.
The next dose is determined to be the next lower dose if $\hat{\pi}(d_j) \geq \lambda_\mathrm{d}$ or the next higher dose if $\hat{\pi}(d_j) \leq \lambda_\mathrm{e}$; otherwise, it is the same dose as the current dose.
Boundaries $\lambda_\mathrm{e}, \lambda_\mathrm{d}$ are derived using the predetermined values $\theta_1, \theta_2$, where $\theta_1$ is the highest probability of toxicity that is considered as underdosing and $\theta_2$ is the lowest probability that is considered as overdosing.
Following Liu and Yuan \cite{liu2015bayesian}, we set  $\theta_1 = 0.6\theta$ and $\theta_2 = 1.4\theta$ in the simulation.

We apply the default dose elimination and stopping rule used in the \texttt{BOIN} R package.
In addition, we also stop the trial if the condition $\mathrm{Pr}(\pi(d_1) \geq \theta \mid D_n) \geq 0.9$ is met.
%

\subsection{BO design}
In the BO design, the dose-toxicity relationship is modeled using \eqref{eq: model}–\eqref{eq: likelihood}, as in the original study \cite{Takahashi2021-MTD}.
As with the proposed design, the doses are set as $(d_1, d_2, d_3, d_4, d_5)=(0.00, 0.25, 0.5, 0.75, 1.00)$.
The dose-finding strategy of the BO design is based on a Bayesian optimization framework.
It assumes $g(d_j) = |\pi(d_j) - \theta|$ as the objective function that should be minimized.
Once the data are obtained, it calculates the value of the acquisition function based on the updated posterior distributions of the probability of toxicity at each dose.
As the acquisition function, the expected improvement (EI) is used, which is given as follows:
\begin{align*}
    \alpha^{\mathrm{EI}}_n(d) = E_g[\max\{0,g^+ -g(d)\}\mid D_n],
\end{align*}
where $g^+ = \min_d E_g[g(d)\mid D_n]$ is the current best value of $g$.
It follows that the value of the EI will be larger at doses that are expected to take even smaller values of $g$ than the current best.
For overdose control, the BO design considers the admissible dose set $A_n'$ and the next dose $x_{n+1}$ is determined as $x_{n+1} = \argmax_{d \in A_n'}\alpha^{\mathrm{EI}}_n(d)$.
The original study assumed four conditions that the doses in $A_n'$ must satisfy, three of which are also used in our design (see $A_n$ in Section \ref{sec: Dose-finding strategy}).
The remaining condition, ``if two or more patients experience DLT, all doses included in $A_n'$ are less than or equal to one dose lower than the current dose,'' is completely based on the rule-based idea such as 3+3 design, and it may lead to inefficient dose selection.
Therefore, we do not consider this condition and use $A_n$ in our design for the BO design in our simulations with $c_1 = 0.5, c_2 = 0.9$.
The BO design terminates the trial when the maximum sample size is reached and estimates the MTD as follows:
\begin{align*}
    \argmax_{d_j\in\{d_j\mid\widehat{\pi}(d_j)<\theta+\delta_2\}} \mathrm{Pr}(\theta - \delta_1 \leq \pi(d_j) \leq \theta + \delta_1 \mid D_N),
\end{align*}
where $\delta_1, \delta_2$ are pre-specified small values.
We set $\delta_1 = 0.05, ~\delta_2 = 0.1$ as in the original study.

For the covariance function, the BO design uses the squared exponential covariance function \eqref{eq: kernel}.
Although $\sigma_f$ is set to 1 in the original study, we assume a weakly informative prior $\log (\sigma_f) \sim N(0.20, 0.45^2)$ to match the condition with our design. 
We set $\ell$ to 1.
To construct a prior mean function, the original study used the systematic approach of Lee and Cheung \cite{Lee2009-la}, whereas in the simulation, the prior mean function of the BO design is generated using the quantile-based approach with the parameters $q_1 = q_J = 0.1$ in our design, as described in Section \ref{sec: prior mean}.

We apply the same safety stopping rule as in the other designs, i.e., we stop the trial if the condition $\mathrm{Pr}(\pi(d_1) \geq \theta \mid D_n) \geq 0.9$ is met.

\clearpage
\section{Acquisition functions for LSE design}\label{appendix: Acquisition functions}
\setcounter{table}{0}
\setcounter{figure}{0}
In Section \ref{sec: Dose-finding strategy}, we proposed an acquisition function based on the expected misclassification probability:
\begin{align*}
    \alpha_n^\mathrm{Mis}(d) := p_n(d)^r \min\{p_n(d),1-p_n(d)\}.
\end{align*}
In addition to this acquisition function, various other acquisition functions can be considered and used as the acquisition function in the LSE design.
Among the various types of acquisition functions, this section discusses the classification ambiguity acquisition function \cite{Gotovos2013} described in Section \ref{sec: Theory ALSE algorithm} and presents a comparison of its performance with $\alpha_n^\mathrm{Mis}$.
To control overdosing, as described in Section \ref{sec: Dose-finding strategy}, weights are introduced into the classification ambiguity acquisition function:
\begin{align*}
    \alpha_n^\mathrm{Amb}(d) := p_n(d)^r \min\{U_n(d)-\theta,\theta-L_n(d)\},
\end{align*}
where $[L_n(d),U_n(d)]$ is a credible interval of $\pi(d)$ after obtaining data from $n$ patients and $r \geq 0$ is  a parameter with the same role as that in $\alpha_n^\mathrm{Mis}$.
This acquisition function is also used in many LSE problems.
$\alpha_n^\mathrm{Mis}$ and $\alpha_n^\mathrm{Amb}$ have a similar form; the only difference is that the former uses the information of the posterior distribution of $\pi$ as the area (posterior probability), whereas the latter uses it as the line segment (credible interval).

To evaluate the performance of the LSE design with these two acquisition functions, we performed simulations with the same settings as those mentioned in Section \ref{sec: Simulations}.
The value of the parameter $r$ was set to 1 for both acquisition functions.
The simulation results are shown in Figures \ref{fig: result_C_acq} and \ref{fig: result_O_acq}.

\begin{figure}[t]
\begin{center}
\includegraphics[width=0.75\linewidth]{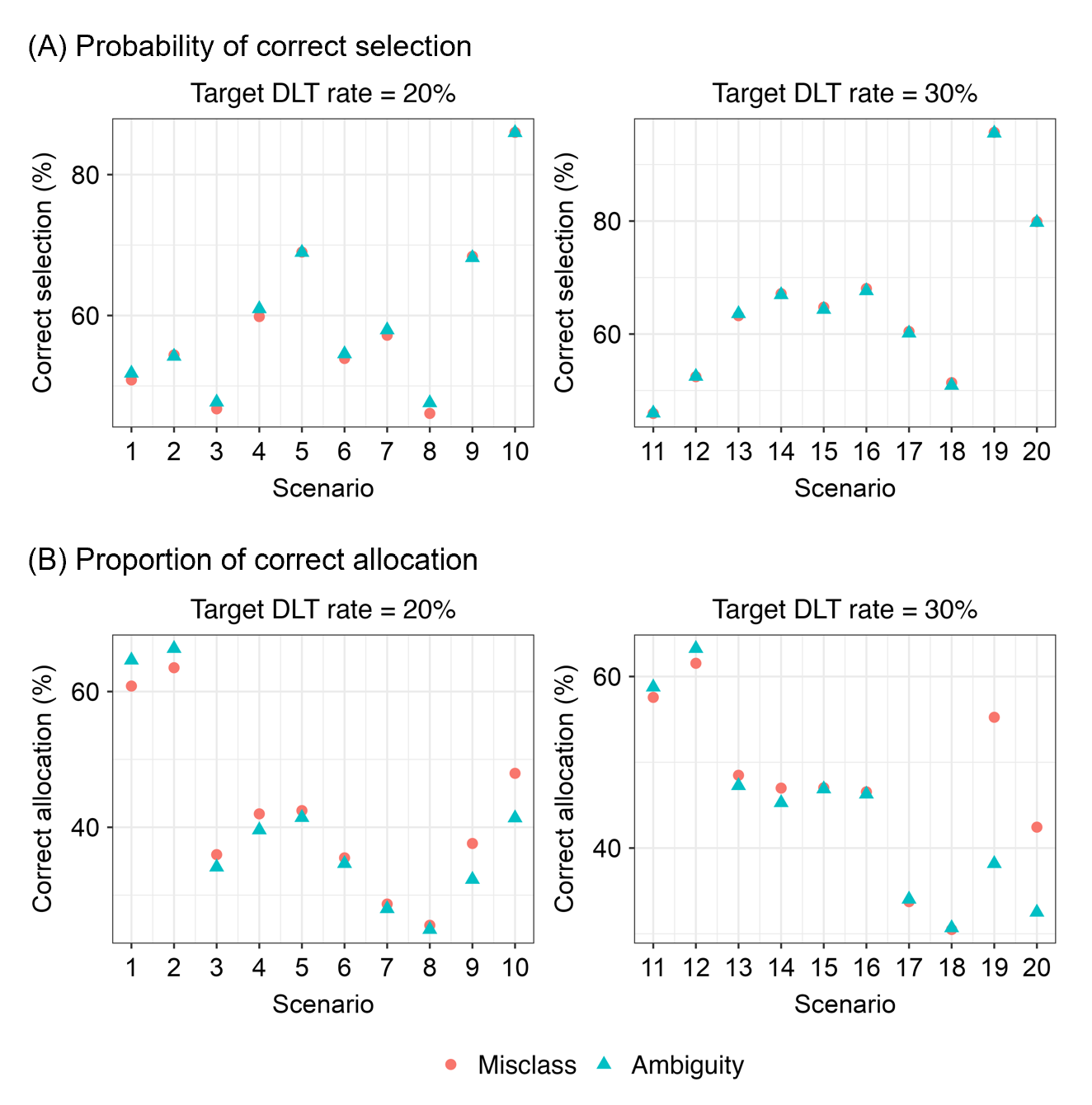}
\caption{Comparison of accuracy for identifying and allocating the MTD for the acquisition functions $\alpha_n^\mathrm{Mis}$ (Misclass) and $\alpha_n^\mathrm{Amb}$ (Ambiguity), with PCS (A) and PCA (B) obtained from 2000 simulations.}
\label{fig: result_C_acq}
\end{center}
\end{figure}

\begin{figure}[t]
\begin{center}
\includegraphics[width=0.75\linewidth]{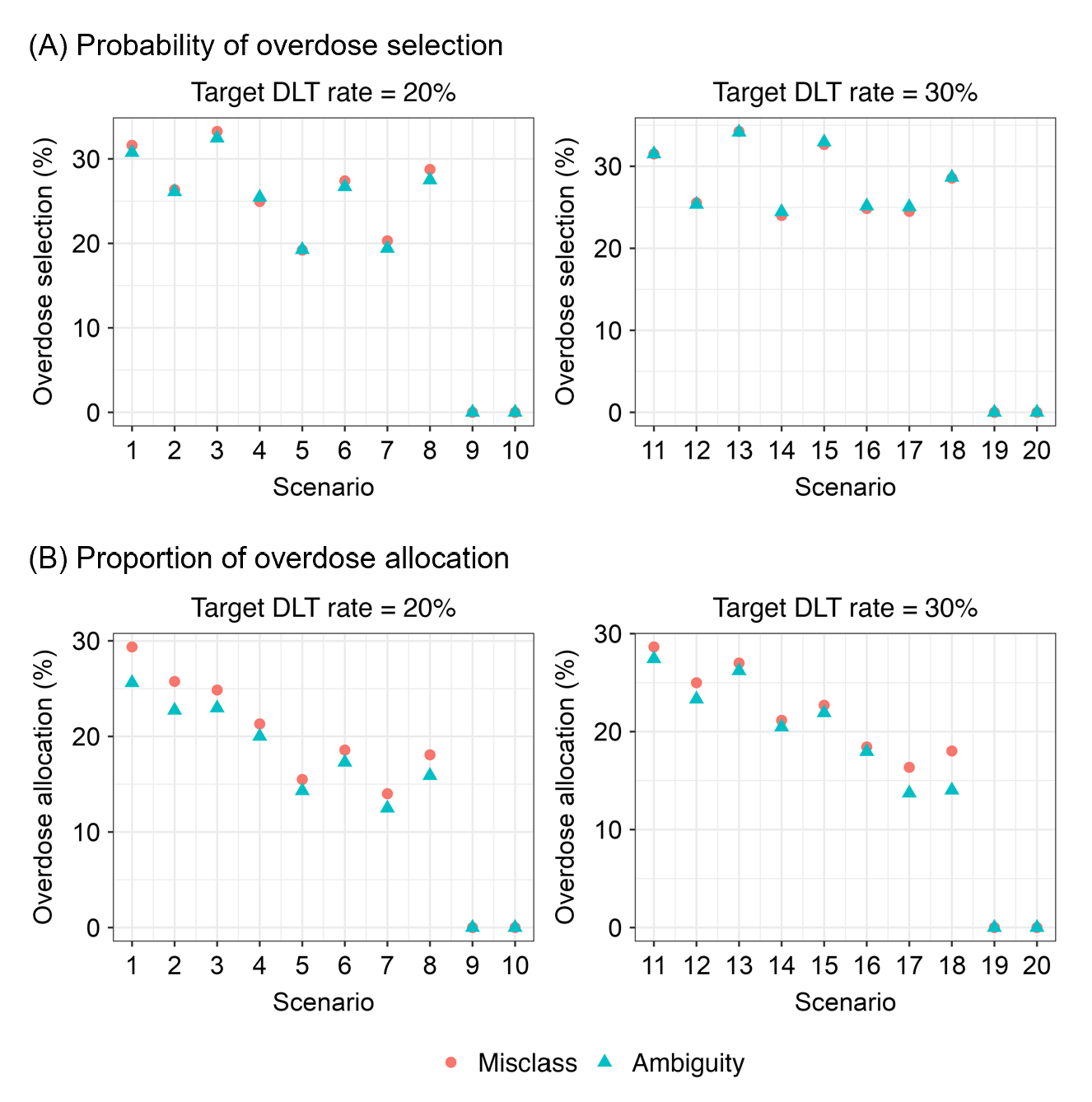}
\caption{Comparison of the risk of overdosing for the acquisition functions $\alpha_n^\mathrm{Mis}$ (Misclass) and $\alpha_n^\mathrm{Amb}$ (Ambiguity), with POS (A) and POA (B) obtained from 2000 simulations.}
\label{fig: result_O_acq}
\end{center}
\end{figure}

\clearpage
\section{Simulation results}\label{appendix: results}
\setcounter{table}{0}
\setcounter{figure}{0}

\begin{threeparttable}[htbp]
\centering 
\caption{Complete results of the simulations in Scenarios 1–10.}
\begin{tabular}{llcccccccccc}
\hline
       &          & \multicolumn{10}{c}{Scenario}                                                 \\ \cline{3-12} 
       & Design   & 1     & 2     & 3     & 4     & 5     & 6     & 7     & 8     & 9     & 10    \\ \hline
PCS\tnote{1}    & LSE($r=0$) & 50.05 & 52.80 & 42.30 & 56.90 & 62.30 & 48.55 & 56.10 & 43.65 & 69.10 & 82.45 \\
       & LSE($r=1$) & 50.85 & 54.40 & 46.75 & 59.85 & 69.00 & 53.90 & 57.20 & 46.10 & 68.40 & 86.00 \\
       & BO       & 49.50 & 53.50 & 44.45 & 57.45 & 66.40 & 51.00 & 56.65 & 46.10 & 69.35 & 84.65 \\
       & CRM      & 42.30 & 46.00 & 43.80 & 58.10 & 63.05 & 47.60 & 56.30 & 47.95 & 75.35 & 86.65 \\
       & BOIN     & 38.95 & 44.35 & 40.70 & 52.95 & 57.50 & 41.20 & 50.15 & 38.90 & 62.40 & 73.55 \\ \hline
PCA\tnote{2}    & LSE($r=0$) & 53.54 & 56.75 & 36.18 & 41.78 & 42.62 & 36.12 & 31.42 & 27.48 & 42.75 & 52.96 \\
       & LSE($r=1$) & 60.81 & 63.51 & 35.97 & 41.98 & 42.46 & 35.50 & 28.69 & 25.56 & 37.61 & 47.95 \\
       & BO       & 57.58 & 60.80 & 37.20 & 42.88 & 43.54 & 36.79 & 29.85 & 26.28 & 39.89 & 50.47 \\
       & CRM      & 45.00 & 47.69 & 36.18 & 43.03 & 42.78 & 36.14 & 32.37 & 31.37 & 46.09 & 53.79 \\
       & BOIN     & 47.63 & 50.07 & 33.89 & 40.46 & 37.50 & 31.10 & 29.19 & 25.33 & 34.64 & 41.94 \\ \hline
POS\tnote{3}    & LSE($r=0$) & 33.10 & 27.75 & 37.00 & 25.45 & 21.10 & 31.55 & 20.50 & 31.90 & 0.00  & 0.00  \\
       & LSE($r=1$) & 31.60 & 26.35 & 33.25 & 24.95 & 19.20 & 27.40 & 20.30 & 28.75 & 0.00  & 0.00  \\
       & BO       & 32.90 & 27.10 & 34.85 & 25.90 & 19.40 & 29.40 & 20.55 & 30.00 & 0.00  & 0.00  \\
       & CRM      & 35.45 & 30.70 & 37.55 & 27.15 & 25.15 & 36.50 & 22.75 & 33.50 & 0.00  & 0.00  \\
       & BOIN     & 29.25 & 23.65 & 29.45 & 21.25 & 17.40 & 26.70 & 18.05 & 26.20 & 0.00  & 0.00  \\ \hline
POA\tnote{4}    & LSE($r=0$) & 36.79 & 32.46 & 31.49 & 28.10 & 21.72 & 24.48 & 18.31 & 22.61 & 0.00  & 0.00  \\
       & LSE($r=1$) & 29.36 & 25.75 & 24.85 & 21.31 & 15.50 & 18.57 & 14.00 & 18.08 & 0.00  & 0.00  \\
       & BO       & 32.60 & 28.38 & 26.84 & 23.97 & 17.70 & 20.50 & 15.22 & 19.86 & 0.00  & 0.00  \\
       & CRM      & 39.08 & 35.96 & 33.91 & 29.87 & 25.64 & 30.25 & 19.21 & 24.04 & 0.00  & 0.00  \\
       & BOIN     & 30.31 & 27.76 & 24.50 & 21.14 & 17.83 & 20.32 & 13.17 & 16.35 & 0.00  & 0.00  \\ \hline
p(DLT)\tnote{5} & LSE($r=0$) & 21.56 & 21.53 & 19.15 & 20.01 & 18.11 & 17.03 & 15.18 & 14.34 & 11.74 & 11.87 \\
       & LSE($r=1$) & 20.60 & 20.73 & 17.84 & 18.31 & 16.26 & 15.78 & 13.72 & 13.32 & 11.03 & 11.07 \\
       & BO       & 20.92 & 20.98 & 18.35 & 19.07 & 17.11 & 16.37 & 14.20 & 13.71 & 11.34 & 11.49 \\
       & CRM      & 20.46 & 20.56 & 19.25 & 20.19 & 18.79 & 18.11 & 15.16 & 15.04 & 12.05 & 11.80 \\
       & BOIN     & 18.33 & 18.51 & 17.20 & 17.76 & 16.03 & 15.22 & 13.41 & 12.99 & 10.49 & 10.06 \\ \hline
\end{tabular}
\begin{tablenotes}
\item[1] PCS=Probability of correct selection
\item[2] PCA=Proportion of correct allocation
\item[3] POS=Probability of overdose selection
\item[4] PCA=Proportion of overdose allocation
\item[5] p(DLT)=Proportion of patients who experienced DLT
\end{tablenotes}
\label{tab: result_1}
\end{threeparttable}

\begin{threeparttable}[htbp]
\centering 
\caption{Complete results of the simulations in Scenarios 11–20.}
\begin{tabular}{llcccccccccc}
\hline
       &          & \multicolumn{10}{c}{Scenario}                                                 \\ \cline{3-12} 
       & Design   & 11    & 12    & 13    & 14    & 15    & 16    & 17    & 18    & 19    & 20    \\ \hline
PCS\tnote{1}    & LSE($r=0$) & 41.90 & 50.65 & 55.40 & 62.45 & 53.90 & 61.15 & 56.15 & 48.75 & 93.30 & 80.80 \\
       & LSE($r=1$) & 45.95 & 52.45 & 63.25 & 67.15 & 64.75 & 68.05 & 60.45 & 51.40 & 95.70 & 79.90 \\
       & BO       & 42.85 & 50.50 & 57.50 & 63.80 & 57.80 & 64.40 & 59.75 & 51.05 & 94.10 & 80.65 \\
       & CRM      & 38.80 & 50.25 & 52.05 & 61.15 & 44.00 & 56.40 & 55.30 & 45.40 & 93.85 & 84.10 \\
       & BOIN     & 33.80 & 42.55 & 51.15 & 54.70 & 47.00 & 56.00 & 53.50 & 43.85 & 83.65 & 76.90 \\ \hline
PCA\tnote{2}    & LSE($r=0$) & 49.49 & 53.09 & 44.81 & 46.27 & 44.13 & 44.74 & 35.19 & 32.57 & 59.03 & 48.10 \\
       & LSE($r=1$) & 57.57 & 61.54 & 48.48 & 46.99 & 47.03 & 46.55 & 33.75 & 30.50 & 55.25 & 42.43 \\
       & BO       & 53.82 & 57.05 & 49.45 & 48.65 & 48.00 & 47.50 & 35.07 & 32.32 & 57.68 & 43.94 \\
       & CRM      & 48.58 & 54.77 & 38.13 & 42.48 & 33.98 & 39.86 & 36.03 & 31.95 & 61.80 & 51.77 \\
       & BOIN     & 39.38 & 44.37 & 42.81 & 43.26 & 36.55 & 39.33 & 33.93 & 30.48 & 50.95 & 43.53 \\ \hline
POS\tnote{3}    & LSE($r=0$) & 35.95 & 27.75 & 38.80 & 26.05 & 40.70 & 29.30 & 28.15 & 33.05 & 0.00  & 0.00  \\
       & LSE($r=1$) & 31.50 & 25.55 & 34.25 & 24.00 & 32.65 & 24.85 & 24.50 & 28.55 & 0.00  & 0.00  \\
       & BO       & 34.60 & 27.45 & 37.40 & 25.75 & 37.20 & 26.80 & 25.95 & 30.70 & 0.00  & 0.00  \\
       & CRM      & 38.75 & 26.35 & 42.20 & 27.55 & 50.55 & 34.20 & 34.00 & 42.35 & 0.00  & 0.00  \\
       & BOIN     & 32.20 & 22.85 & 30.90 & 22.15 & 36.80 & 23.85 & 24.90 & 29.85 & 0.00  & 0.00  \\ \hline
POA\tnote{4}    & LSE($r=0$) & 36.76 & 33.47 & 35.92 & 28.95 & 30.83 & 26.77 & 22.40 & 24.03 & 0.00  & 0.00  \\
       & LSE($r=1$) & 28.64 & 24.99 & 27.00 & 21.15 & 22.69 & 18.43 & 16.35 & 18.02 & 0.00  & 0.00  \\
       & BO       & 32.38 & 29.46 & 30.15 & 24.13 & 25.25 & 21.42 & 17.95 & 19.83 & 0.00  & 0.00  \\
       & CRM      & 37.40 & 30.96 & 42.45 & 32.10 & 41.89 & 34.06 & 26.86 & 31.72 & 0.00  & 0.00  \\
       & BOIN     & 34.43 & 29.22 & 29.00 & 23.92 & 29.22 & 23.02 & 18.73 & 20.39 & 0.00  & 0.00  \\ \hline
p(DLT)\tnote{5} & LSE($r=0$) & 29.04 & 30.03 & 28.94 & 29.57 & 26.16 & 27.23 & 23.40 & 22.77 & 19.02 & 18.96 \\
       & LSE($r=1$) & 28.23 & 28.92 & 26.90 & 27.23 & 24.34 & 24.87 & 21.38 & 21.10 & 18.12 & 17.84 \\
       & BO       & 28.51 & 29.44 & 28.06 & 28.43 & 25.33 & 26.03 & 22.07 & 21.76 & 18.73 & 18.11 \\
       & CRM      & 28.89 & 29.87 & 29.80 & 30.14 & 27.15 & 28.98 & 24.75 & 24.37 & 19.73 & 19.64 \\
       & BOIN     & 24.96 & 25.92 & 26.26 & 27.15 & 23.80 & 25.08 & 22.02 & 21.68 & 17.22 & 17.92 \\ \hline
\end{tabular}
\begin{tablenotes}
\item[1] PCS=Probability of correct selection
\item[2] PCA=Proportion of correct allocation
\item[3] POS=Probability of overdose selection
\item[4] PCA=Proportion of overdose allocation
\item[5] p(DLT)=Proportion of patients who experienced DLT
\end{tablenotes}
\label{tab: result_2}
\end{threeparttable}

\begin{figure}[htbp]
\begin{center}
\includegraphics[width=0.8\linewidth]{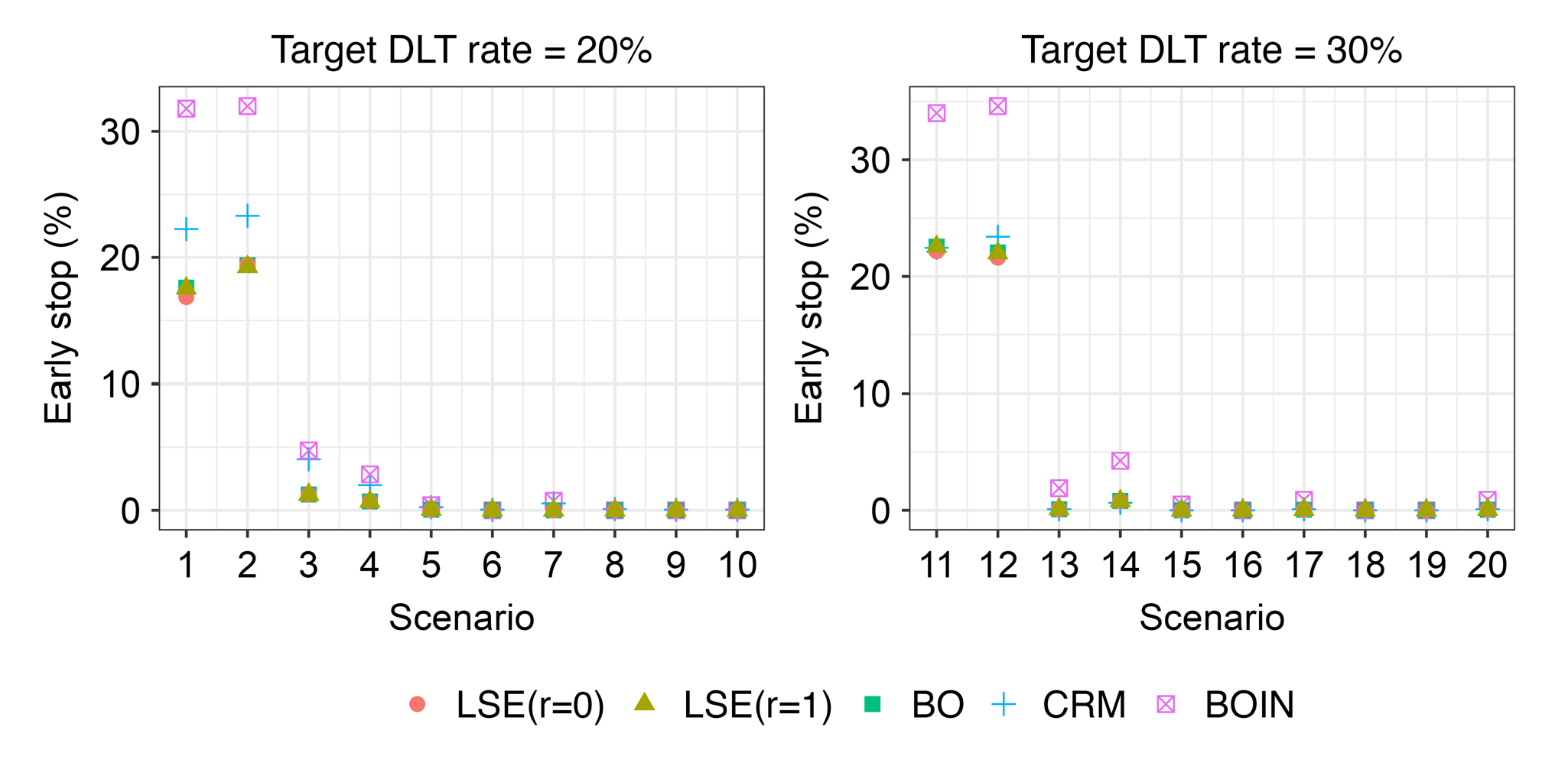}
\caption{Proportion of simulated trials that were terminated early owing to the safety stopping rule.}
\label{fig: Early stop}
\end{center}
\end{figure}

\begin{figure}[htbp]
\begin{center}
\includegraphics[width=0.8\linewidth]{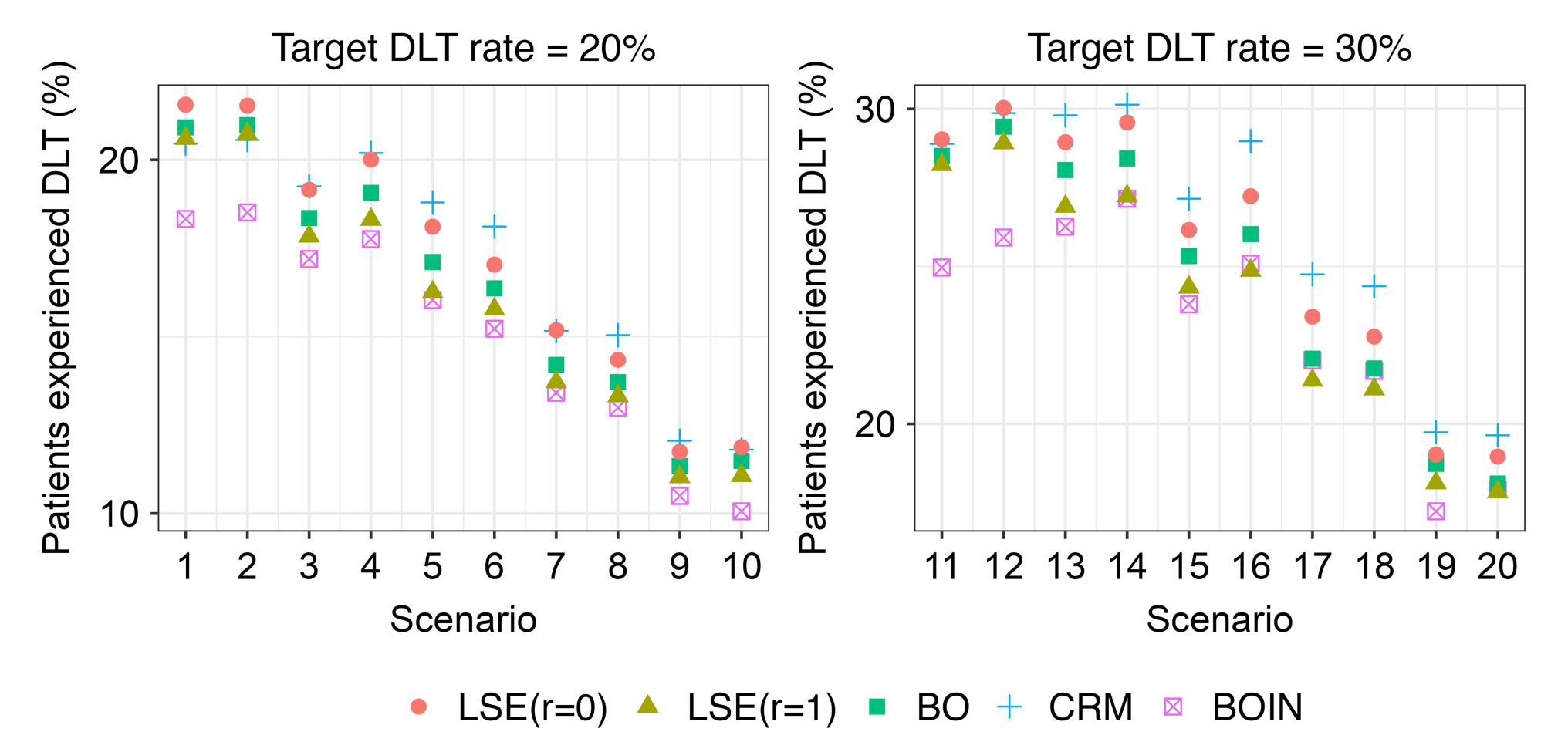}
\caption{Average proportion of patients who experienced DLT within each simulated trial across 2000 simulations.}
\label{fig: DLT_p}
\end{center}
\end{figure}

\end{document}